\documentclass{article}

\usepackage{fullpage}
\usepackage[round]{natbib}

\usepackage[utf8]{inputenc}
\usepackage[T1]{fontenc}
\usepackage[mathscr]{eucal}  %
\usepackage{inconsolata}
\usepackage{microtype}
\usepackage{algorithm}
\usepackage[noend]{algpseudocode}
\usepackage{amsmath, amssymb, amsthm}
\usepackage{bm}
\usepackage{graphicx}
\usepackage{thmtools}
\usepackage{thm-restate}
\usepackage{hyperref} %
\usepackage[capitalize, nameinlink]{cleveref}
\usepackage{comment}
\usepackage{enumerate}
\usepackage{mathtools}
\usepackage{subcaption}
\usepackage{nag}
\usepackage{stmaryrd}
\usepackage{parskip}
\usepackage{upgreek}
\usepackage{todonotes}

\graphicspath{{figures/}{.}}

\newcommand{\declarecolor}[2]{\definecolor{#1}{RGB}{#2}\expandafter\newcommand\csname #1\endcsname[1]{\textcolor{#1}{##1}}}
\declarecolor{White}{255, 255, 255}
\declarecolor{Black}{0, 0, 0}
\declarecolor{LightGray}{216, 216, 216}
\declarecolor{Gray}{127, 127, 127}
\declarecolor{Orange}{237, 125, 49}
\declarecolor{LightOrange}{251,229, 214}
\declarecolor{Yellow}{255, 192, 0}
\declarecolor{LightYellow}{255, 242, 200}
\declarecolor{Blue}{91, 155, 213}
\declarecolor{LightBlue}{222, 235, 247}
\declarecolor{Green}{112, 173, 71}
\declarecolor{LightGreen}{226, 240, 217}
\declarecolor{Navy}{68, 114, 196}
\declarecolor{LightNavy}{218, 227, 243}

\hypersetup{
	colorlinks=true,
	pdfpagemode=UseNone,
	citecolor=Navy,
	linkcolor=Navy,
	urlcolor=Navy,
}

\crefformat{equation}{(#2#1#3)}

\DeclareMathOperator*{\argmax}{arg\,max}

\renewcommand{\emptyset}{\varnothing}  %

\theoremstyle{plain}
\newtheorem{theorem}{Theorem}[section]
\newtheorem{lemma}[theorem]{Lemma}

\theoremstyle{definition}
\newtheorem{definition}[theorem]{Definition}

\newcommand{\onev}{\mathbf{1}}
\newcommand{\zerov}{\mathbf{0}}

\newcommand{\vtheta}{\boldsymbol{\mathit{\theta}}}

\renewcommand{\aa}{\boldsymbol{\mathit{a}}}
\newcommand{\bb}{\boldsymbol{\mathit{b}}}

\newcommand{\oyy}{\boldsymbol{\mathit{\bar{y}}}}

\newcommand{\rr}{\boldsymbol{\mathit{r}}}

\newcommand{\ww}{\boldsymbol{\mathit{w}}}

\newcommand{\xx}{\boldsymbol{\mathit{x}}}

\newcommand{\yy}{\boldsymbol{\mathit{y}}}

\renewcommand{\AA}{\boldsymbol{\mathit{A}}}

\newcommand{\oHH}{\boldsymbol{\overline{\mathit{H}}}}

\newcommand{\XX}{\boldsymbol{\mathit{X}}}

\renewcommand{\epsilon}{\varepsilon}

\usepackage[T1]{fontenc}    %
\usepackage{booktabs}       %
\usepackage{amsmath}
\usepackage{amsfonts}       %
\usepackage{nicefrac}       %
\usepackage{microtype}      %
\usepackage{xcolor}         %
\usepackage{algpseudocode}
\usepackage{paralist}
\usepackage{cleveref}

\usepackage{hyperref}
\usepackage{url}
\DeclareMathOperator*{\argsort}{arg\,sort}

\title{Greedy PIG: Adaptive Integrated Gradients}

\author{%
Kyriakos Axiotis$^*$,
Sami Abu-al-haija$^*$,
Lin Chen,
Matthew Fahrbach,
Gang Fu \\
Google Research\\
\texttt{\{axiotis, haija, linche, fahrbach, thomasfu\}@google.com} \\
}
\date{}

\begin{document}

\maketitle
\def\thefootnote{*}\footnotetext{Equal contribution.}\def\thefootnote{\arabic{footnote}}

\begin{abstract}
Deep learning has become the standard approach for most machine learning tasks. While its impact is undeniable, interpreting the predictions of deep learning models from a human perspective remains a challenge. In contrast to model training, model interpretability is harder to quantify and pose as an explicit optimization problem. Inspired by the AUC softmax information curve (AUC SIC) metric for evaluating feature attribution methods, we propose a unified discrete optimization framework for feature attribution and feature selection based on subset selection.
This leads to a natural \emph{adaptive} generalization of the path integrated gradients (PIG) method for feature attribution, which we call Greedy PIG.
We demonstrate the success of Greedy PIG on a wide variety of tasks,
including image feature attribution, graph compression/explanation, and
post-hoc feature selection on tabular data.
Our results show that introducing adaptivity is a powerful and versatile method for
making attribution methods more powerful.
\end{abstract}

\section{Introduction}

Deep learning sets state-of-the-art on a wide variety of machine learning tasks,
including image recognition, natural language understanding, large-scale recommender systems, and generative %
models~\citep{graves2013speech,he2016deep,krizhevsky2017imagenet,bubeck2023sparks}.
Deep learning models, however, are often opaque and hard to interpret.
There is no native procedure for \textit{explaining} model decisions to humans,
and explainability is essential when models make
decisions that directly influence people's lives,
\textit{e.g.}, in healthcare, epidemiology, legal, and education \citep{ahmad2018interpretable,wiens2018machine,chen2021ethical, abebe2022adversarial, abebe2023students}.

We are interested in two routines: (i) \emph{feature attribution} and (ii) \emph{feature selection}.
Given an input example, feature attribution techniques
offer an explanation of the model's decision
by assigning an \textit{attribution score} to each input feature.
These scores can then be directly rendered in the input space, \textit{e.g.}, as a 
heatmap highlighting visual saliency \citep{sundararajan2017axiomatic},
offering
an explanation of model decisions
that is human-interpretable.
On the other hand, feature selection methods
find the set of most-informative features for a machine learning model across many examples
to optimize its prediction quality.
While the goal of both feature attribution and feature
selection is to find the most impactful features, they are
inherently different tasks with a disjoint set of approaches and literature.
Concretely, feature attribution considers one example per invocation,
while feature selection considers an entire dataset.
For literature surveys,
see \citep{zhang2021survey} for feature attribution and 
interpretability see
and \citep{li2017feature} for feature selection.

Our main contributions are as follows:
\begin{compactitem}
\item We consider the problem of feature attribution from first principles and, inspired by evaluation metrics in the literature~\citep{kapishnikov2019xrai}, we propose a new formulation of feature attribution as an explicit subset selection problem.
\item Using this subset selection formulation, we ascribe the main shortcomings of the \emph{path integrated gradient}  (PIG) algorithms to their limited ability to handle feature correlations. This motivates us to introduce adaptivity to these algorithms, which is a natural way to account for correlations, leading to a new algorithm that we call Greedy PIG.
\item We leverage the generality of the subset selection formulation to propose different objectives for feature attribution based on model loss that can better capture the model's behavior, as well as a unification between feature attribution and feature selection formulations.
\item We evaluate the performance of Greedy PIG on a variety of experimental tasks, including feature attribution, 
graph neural network (GNN) compression, and post-hoc feature selection. Our results show that adaptivity is a powerful ingredient for improving the quality of deep learning model interpretability methods.
\end{compactitem}

\section{Preliminaries}

\subsection{Path integrated gradients (PIG)}
Let $f(\cdot; \vtheta)$ be a pre-trained neural network with parameters $\vtheta$.
Our work assumes that $f$ is pre-trained and therefore, we often omit $\vtheta$ for brevity.
Suppose $\xx \in \mathbb{R}^n$ is an input example.
The \textit{path integrated gradients} (PIG) method of \citet{sundararajan2017axiomatic} attributes the decision of $f$ given $\xx$,
by considering a \emph{line path} from the ``baseline''\footnote{\citet{sundararajan2017axiomatic} uses a black-image and \citet{kapishnikov2019xrai} averages the integral over black- and white-images. Random noise baselines have also been considered.%
} example $\xx^{(0)}$ to $\xx$:
\begin{equation}
\label{eq:pig}
\gamma : [0, 1] \rightarrow \mathbb{R}^n \ \ \ \ \ \textrm{with} \ \ \ \ \ \gamma(t; \xx^{(0)}, \xx) = \xx^{(0)} + t (\xx -  \xx^{(0)})  \ \ \ \ \ \textrm{for} \ \ \ \ t \in [0, 1].
\end{equation}
Then,
the top features $i\in [n]$ maximizing the following
(weighted) integral are chosen:
\begin{equation}
\label{eq:pig_argsort}
  \argsort_{i \in [n]} \ \ \  \left[  \ \ \  (\xx_{[i]} -  \xx^{(0)}_{[i]}) \int_0^1  \frac{\partial}{\partial \xx_{[i]}} f(\gamma(t) )  dt  \ \ \  \right].
\end{equation}
Our work addresses a weakness of PIG, which is 
that it is a one-shot algorithm.
Specifically, an invocation to PIG computes (in-parallel) Eq.~\ref{eq:pig} and chooses (at-once) the indices $\subseteq [n]$ maximizing Eq.~\ref{eq:pig_argsort}.

\subsection{Subset selection}
\label{sec:prelim_subset_selection}
Subset selection is a family of discrete optimization problems,
where the goal is to select a subset $S \subseteq [n] := \{1, 2, \dots, n\}$
that maximizes a set function $G:2^{[n]} \rightarrow \mathbb{R}$
subject to a cardinality constraint $k$:
\begin{align}
    \label{eq:framework}
    S^*
    =
    \argmax_{S \subseteq [n], |S| \le k} G(S).
\end{align}
Prominent examples of subset selection in machine learning include
data summarization~\citep{kleindessner2019fair}, feature selection~\citep{cai2018feature,elenberg2018restricted,bateni2019categorical,chen2021feature,yasuda2023sequential},
and submodular optimization~\citep{krause2014submodular,mirrokni2015randomized,fahrbach2019non,fahrbach2019submodular}.

Throughout our work,
we let $\onev$ denote $n$-dimensional all-one vector, $\onev_S \in \{0, 1\}^n$ be the indicator vector of~$S$,
and let
$\xx_S = \onev_S \odot \xx \in \mathbb{R}^n$ where $\odot$ denotes Hadamard product.

\section{Our method}
\subsection{Motivation: Discrete optimization for feature attribution and selection}
\label{sec:subset_selection}
The general subset selection framework (\S\ref{sec:prelim_subset_selection}) allows recovering several tasks.
For (i) feature attribution,
the \emph{softmax information curve (SIC)} of \cite{kapishnikov2019xrai} can be recovered from
(Eq.~\ref{eq:framework})
by setting $G(S)$ to the 
softmax output of a target class (see Eq.~\ref{eq:attribution_softmax}).
For (ii) feature selection,
we can set $G(S)$ to maximum log likelihood achievable by a model trained on a subset
of features $S$.
Finally, if one sets $G(S)$ to be
the maximum model output perturbation achieved by only
changing the values of features in $S$, one
recovers a task of finding \emph{counterfactual
explanations}.

The generality of framework (\S\ref{sec:prelim_subset_selection}) encourages us to pose the tasks (i) and (ii) as instances.
Consequently, we inherit
algorithms and intuitions from the subset
selection and submodular optimization literature.
The area of submodular optimization has a vast literature with different theoretical analyses and algorithms based on the specific properties of the set function. For example, for maximization of weakly submodular functions, it is known that the greedy subset selection algorithm achieves a constant approximation ratio. In general, determining the right realistic assumptions on the set function is a major open problem in the area.
As such, it is not yet clear which of these assumptions are realistic for each use case.
We believe this is a very important question for future work.

\subsection{Instantiations}
\paragraph{Multi-class attribution.}
Suppose that 
classifier $f$ is the softmax output of a multi-class neural network,
then
the softmax information curve of~\cite{kapishnikov2019xrai} 
can be written in (\S\ref{sec:attribution_subset_selection}; Eq.~\ref{eq:framework}) as:
\begin{equation}
\label{eq:attribution_softmax}
    G^\textsc{AttributionSoftmax}(S) =  f_{c^*}(\xx_S)\,, \ \ \ \ \textrm{with} \ \ \ \ c^* = \mathop{\arg\max}_j f_j(\xx).
\end{equation}
Maximizing $S$ in Eq.~\ref{eq:attribution_softmax} chooses feature subset $S$ that maintains the model's decision, as compared to the fully-observed example $\xx$.
Crucially, however, the top class $c^*$ doesn't tell the whole story, since the
input might be ambiguous or the classifier might be uncertain.
It is then natural that one would want to capture the behavior of $f$ on
\emph{all} classes, rather than just the top one.
A natural alternative is to re-use the objective~($\ell$)
that the model was trained with, and incorporate it into a subset  selection
problem.
Specifically, we define the set function:
\begin{equation}\label{eq:g_attribution_kl}
    G^\textsc{AttributionKL}(S) =  \ell(f(\xx), f(\xx_S)) = f(\xx)  \log f(\xx_S) \,, %
\end{equation}
where, for classification, $\ell$ is the log-likelihood.
This quantifies the similarity of probability distribution $f(\xx_S)$, i.e., considering a subset of features,
with the distribution $f(\xx)$, i.e., considering all features.
Maximizing $G(S)$ under a cardinality constraint is then equivalent
to seeking a small number of features that capture the multiclass
behavior of the model on a fixed input example.

\paragraph{Feature selection.}
Unlike feature attribution,
outputting attributions 
for \emph{a single example},
in feature selection,
the goal is to select a \emph{global set} of features for \emph{all dataset examples},
that maximize the predictive capability of the model.
Feature selection can be formulated by defining 
\begin{align}
    G^\textsc{FeatSelect}(S) = \max_{\vtheta} \ell(\yy, f(\XX_S; \vtheta))\,,
    \label{eq:feature_selection}
\end{align}
i.e., the maximum log likelihood
obtainable by training the model $f$ on the whole dataset $\XX$, given only
the features in $S$. Such formulation of feature selection
as subset selection has been
studied extensively, e.g.,
\citep{liberty2017greedy,elenberg2018restricted,bateni2019categorical,chen2021feature,yasuda2023sequential,axiotis2023performance}.
To avoid re-training $f$,
we consider a simpler problem
of \emph{post-hoc feature selection}, with $G^\textsc{FeatSelectPH}(S) = \ell(\yy, f(\XX_S; \vtheta))$ where the model parameters $\vtheta$ are pre-trained on all features.%
We leave further investigation of
(\ref{eq:feature_selection}) to future work.
The post-hoc formulation
has the advantage of making the objective differentiable,
and thus directly amenable to gradient-based approaches, while also not requiring 
access to the training process of $f$, which is seen as a black box.
This can be particularly useful for applications with limited resources
or engineering constraints that discourage running or modifying the training
pipeline. As we will show in \S\ref{sec:feature_selection},
high quality feature selection can be performed even in this restricted setting.

\subsection{Continuous extension}
To make the problem
(Eq.~\ref{eq:framework}) amenable to continuous optimization methods,
we rely on a continuous extension $g:[0,1]^n\rightarrow\mathbb{R}$
with
$G(S) := g(\onev_S)$. The extension can be derived by replacing the invocation of $f$ on $\xx_S$ or $\XX_S$ in Eq.~\ref{eq:attribution_softmax}--\ref{eq:feature_selection} by an invocation on a \textit{path},
similar to Eq.~\ref{eq:pig}, but with domain equal to the $n$-hypercube:
\begin{equation}
\label{eq:g}
g : [0, 1]^n \rightarrow \mathbb{R}^n \ \ \ \ \ \textrm{with} \ \ \ \ \     g(\mathbf{s}; \xx^{(0)}, \xx ) = \xx^{(0)} + \mathbf{s} \odot (\xx - \xx^{(0)}) \ \ \ \ \ \textrm{for} \ \ \ \ \mathbf{s} \in [0, 1]^n.
\end{equation}
This makes it easy to adapt
feature attribution methods 
from the literature to our framework and
is a lightweight assumption since
loss functions in deep learning are continuous by definition. 
In particular, the integrated gradients 
algorithm of \cite{sundararajan2017axiomatic}
can simply be stated as
computing the vector of feature scores given
by $\int_{t=0}^1 \nabla g(t\onev) dt$.

\subsection{Greedy path integrated gradients (Greedy PIG)}
We now move to
identify and improve the weaknesses of the integrated gradients algorithm.
In contrast to previous works that focus on reducing the amount of
``noise'' perceived by humans in the attribution result,\footnote{
In fact, as noticed in~\cite{smilkov2017smoothgrad}, a noisy 
attribution output is not necessarily bad, since our goal is to
\emph{explain model predictions} and not produce a human-interpretable
understanding of the \emph{data}, which is an
important but different task.
Neural net predictions do not have to conform
to human understanding of content.} we 
identify a different weakness of PIG---it is \emph{sensitive to feature correlations}.
In \S\ref{sec:attribution_subset_selection}, we show that even in
simple linear regression settings, PIG
exhibits pathologies such as outputting attribution scores that are negative,
or whose ordering can be manipulated via feature replication. This leads
us to introduce adaptivity to the algorithm.
We are motivated by the greedy algorithm for subset selection, which is known to provably tackle 
submodular optimization problems, even when there are strong correlations
between features.

\paragraph{Why Greedy captures correlations.}
A remarkable property of the greedy algorithm is that its adaptivity
property can automatically deal with correlations. Consider a set function $G$ and three elements $i, j, k$.
Suppose that greedy initially selects element $i$. Then, it will recurse on the residual
set function $G'(S)=G(S\cup\{i\})$. 
Thus, in the next phases, correlations between $i$ and the remaining elements are eliminated, since $i$ is fixed (selected). In other words, since greedy conditions on the already selected variables, this eliminates their correlations with the rest of the unselected variables.

\paragraph{Greedy PIG.}
Inspired by the adaptivity complexity of submodular maximization~\citep{balkanski2018adaptive}, 
where it has been shown that $\Omega(\log n / \log\log n)$ 
adaptive rounds of evaluating $G$ are needed to achieve provably-better solution
quality than one-shot algorithms (such as integrated gradients),
we propose a 
generalization of integrated gradients with multiple rounds.
We call this algorithm Greedy PIG since it greedily selects the top-attribution
features computed by integrated gradients in each round, \textit{i.e.} the set of $S \subseteq [n]$ maximizing the $\arg\,\textrm{sort}$ of Eq.~\ref{eq:pig_argsort},
and hardcodes their corresponding entries in $\mathbf{s}$ to 1.

\subsection{The greedy PIG algorithm}
\label{sec:greedy_pig}

Our idea is to iteratively compute the top features as attributed by integrated gradients,
and select these features by always including them in future runs
of integrated gradients (and only varying the rest of the variables along the
integration path).
We present this algorithm in detail in Algorithm~\ref{alg:greedy_pig}.

\begin{algorithm}[htb]
\caption{Greedy PIG (path integrated gradients)}
\label{alg:greedy_pig}
\begin{algorithmic}[1]
\State {\bf Input:} access to a gradient oracle for $g:[0,1]^n\rightarrow\mathbb{R}$
\State {\bf Input:} number of rounds $R$, number of selections per round $z$
\State Initialize $S \leftarrow \emptyset$ \hfill \Comment{Selected set of features}
\State Initialize $\aa \leftarrow \zerov$ \hfill \Comment{Vector of attributions}
\For{$r=1$ to $R$}
    \State Set $\hat{\aa} \leftarrow \int_{t=0}^1 \nabla_{\bar{S}} g(\onev_S+t\onev_{\bar{S}}) dt$  \label{alg_line:integral}
    \State Set  $\hat{S} \leftarrow\mathrm{top}_z(\hat{\aa})$ \hfill \Comment{Top-$z$ largest entries of $\hat{\aa}$}
    \State Update $\aa \leftarrow \aa + \hat{\aa}_{\hat{S}}$
    \State Update $S\leftarrow S\cup\hat{S}$
\EndFor
\Return{$\aa, S$}
\end{algorithmic}
\end{algorithm}

In practice, Greedy PIG can be implemented by repeated invocations of PIG (Eq.~\ref{eq:pig}\&\ref{eq:pig_argsort}).
The impact of~$g$ (Eq.~\ref{eq:g}) can be realized by iteratively updating:
\begin{equation}
    \hat{S} \leftarrow \textrm{top indices of Eq.~\ref{eq:pig_argsort} with \ \ }
    \gamma(.;  \xx^{(j+1)},  \xx) \textrm{ \ \ with \ \ } \ \   \xx^{(j+1)} \leftarrow \xx_{\hat{S}} + \xx^{(j)}_{\bar{\hat{S}}}
\end{equation}

Specifically, Greedy PIG is related to a continuous version of Greedy. In the continuous setting, selecting an element $i$ is equivalent to fixing the $i$-th element of the baseline to its final value. The $i$-th element now does not vary and its correlations with the remaining elements are eliminated.
At a high level, our approach has similarities to 
\citep{kapishnikov2021guided},
which adaptively computes the trajectory of the PIG path based on
gradient computations.

Similar to previous work,
the integral in Algorithm~\ref{alg:greedy_pig} on \Cref{alg_line:integral} can
be estimated by discretizing it into a number of steps.
A simpler approach that we found to be competitive and more frugal
is to approximate the integral by a single gradient
$\nabla_{\bar{S}} g(\onev_S)$, which we call the \emph{Sequential Gradient (SG)} algorithm. It should be
noted that, in contrast to integrated gradients,
Algorithm~\ref{alg:greedy_pig} empirically does not return negative attributions.

\section{Analysis}
\subsection{Attribution as subset selection}
\label{sec:attribution_subset_selection}

The integrated gradients method is widely used because of its simplicity and
usefulness. Since its inception,
it was known to satisfy two axioms that~\cite{sundararajan2017axiomatic} 
called sensitivity and implementation
invariance, as well as the convenient property that the sum of attributions is
equal to the function value. However, there is still no deep theoretical 
understanding of when this method (or any attribution method, for that matter)
succeeds or fails, or even how success is defined. 

Previous works have studied a variety of failure cases. In fact, it is well-observed
that the attributions on image datasets are often ``grainy,'' a property deemed
undesirable from a human standpoint (although as stated before, it is not
necessarily bad as long as the attributions are highly predictive of the model's
output). In addition, there will almost always exist some negative 
attribution scores, whose magnitude can be even higher than that of positive
attribution scores.
There have been efforts to mitigate these effects for computer
vision tasks, most notably by ascribing them to noise \citep{smilkov2017smoothgrad,adebayo2018sanity}, 
region grouping
\citep{ancona2017towards,kapishnikov2019xrai}, 
or ascribing them to the magnitude of
the gradient \citep{kapishnikov2021guided}.

In this work, we take a different approach and start from the basic question:
\begin{quote}
    \hspace{2cm} \emph{What are attribution scores supposed to compute?}
\end{quote}
An important step towards answering this question was done in
\citep{kapishnikov2019xrai}, which defined the \emph{softmax information
curve (SIC)} for image classification. In short, the idea is that if we
only keep $k$ features with the top attributions and drop the rest of the
features, the softmax output of the predicted class should be as large as
possible. This is because then, the model prediction can be distilled down
to a small number of features. 

Based on the above intuition, we define the problem in a more general form:
\begin{definition}[Attribution as subset selection]
\label{def:subset_selection}
Given a set function $G:2^{[n]}\rightarrow\mathbb{R}_{>0}$,
a value $k\in [n]$,
and a permutation $\rr=(r_1,r_2,\dots,r_n)$ of $[n]$,
the \emph{attribution quality} of $\rr$ \emph{at level $k$} is defined as
\[
    \frac{G(\{r_1,r_2,\dots,r_k\})} {\max_{S\subseteq [n]: |S|= k} G(S)}.
\]
We define the attribution quality of $\rr$ 
(without reference to $k$)
as the area under the curve defined by
the points $(k, G(\{r_1,r_2,\dots,r_k\})$ for all $k\in [n]$.
\end{definition}

The key advantage of this approach is that we are able to pose
attribution as an explicit optimization problem, and hence we are able to
compare, evaluate, and analyze the quality of attribution methods. In
addition, the formulation in Definition~\ref{def:subset_selection}
gives us flexibility in picking the objective $G$ to be maximized.
In fact, in \S\ref{sec:feature_attribution} we propose setting $G$ to be the log likelihood (\textit{i.e.}, negative cross entropy loss) instead of the top softmax score. 
In \S\ref{sec:feature_selection},
we will see how this formulation allows us to tie attribution together with feature selection.

\paragraph{Note.} As Definition~\ref{def:subset_selection} shows, we are
interested in the problem of attributing
predictions to a \emph{small} number of features. For example, when faced
with redundant features, we wish to pinpoint only a small number of them.
This does not capture applications in which the goal is to find the 
attribution for \emph{all} features, where the goal would be to assign equal
attribution score to all redundant features.

\subsection{Correlations and adaptivity}

Given the problem formulation in Definition~\ref{def:subset_selection}, we
are ready to study the strengths and weaknesses of attribution methods.
First, it is important to define the notion of \emph{marginal gain},
which is a central notion in subset selection problems.
\begin{definition}[Marginal gain]
Given a set function $G:2^{[n]}\rightarrow \mathbb{R}_{\ge0}$, the marginal
gain of the $i$-th element at $S\subseteq [n]$ is defined by
$G(S\cup\{i\}) - G(S)$. When $S$ is omitted, we simply define the 
marginal gain as $G(\{i\}) - G(\emptyset)$.
\end{definition}
Marginal gains are important because they are closely related to
the well-studied greedy algorithm for subset selection~\citep{nemhauser1978analysis},
and are crucial for its analysis. Therefore, a natural question arises:
\emph{Are 
the outputs of attribution methods correlated with the marginal 
gains?}

For this, let us assume that the set function
$G$ is induced by a continuous relaxation
$g:[0,1]^n\rightarrow\mathbb{R}$. We show that the integrated gradient attributions approximate the marginal gains up to an additive term
that depends on the second-order behavior of $g$, and corresponds to
the amount of correlation between variables. In fact, when the variables
are uncorrelated, integrated gradient attributions are equal to the marginal gains.
\begin{lemma}[PIG vs marginal gain]
\label{lem:pig_marginal}
Let $H(\ww)$ be the Hessian of 
a twice continuously differentiable function 
$g:[0,1]^n\rightarrow\mathbb{R}$ at $\ww$. Then, for all $i\in[n]$,
\[
    \left|\int_{t=0}^1 \nabla_i g(t\onev) dt - (g(\onev_i) - g(\zerov))\right| 
    \leq
    \frac{1}{2} \max_{\ww\in[0,1]^n, i\in[n]} \left|\sum_{j\neq i}H_{ij}(\ww)\right|\,,
\]
where $\onev_i$ is the $i$-th standard basis vector of $\mathbb{R}^n$.
\end{lemma}
Lemma~\ref{lem:pig_marginal} tells us that the quality of how well the 
integrated gradient attribution scores approximate the marginal gains
is strongly connected to correlations between input variables. Indeed,
if these correlations are too strong, the attributions can 
diverge in magnitude and sign, even in simple settings. This motivates combining 
integrated gradients with adaptivity (e.g. the Greedy algorithm), since
adaptivity naturally takes correlations into account. The Greedy PIG algorithm, which
we will define in \S\ref{sec:greedy_pig}, is a natural combination of integrated
gradients and Greedy.

We now look at a more concrete failure case of integrated gradients in
Lemma~\ref{lem:feature_redundancy}, that can
arise because of feature redundancy, and how adaptivity can help overcome this issue.
\begin{lemma}[Feature redundancy]
\label{lem:feature_redundancy}
Consider a continuous set 
function $g:[0,1]^n \rightarrow \mathbb{R}$ and~$t$ redundant features
numbered $1,2,\dots,t$, or in other words for any $\ww\in[0,1]^n$ we have
$g((w_1,\dots,w_n)) = h((\max\{w_1,\dots,w_t\},w_{t+1},\dots,w_n))$ for some 
$h:[0,1]^{n-t+1}\rightarrow\mathbb{R}$.
Then, the integrated gradients algorithm with a baseline of $\zerov$
will assign equal attribution score to features
$1$ through $t$.
\end{lemma}
It follows from Lemma~\ref{lem:feature_redundancy} that it is possible to replicate
the top-attributed feature multiple times such that the top 
attributions are all placed on redundant copies of the same feature, therefore missing
features $t+1$ through $n$.
This behavior arises because integrated gradients is a one-shot 
algorithm, but it can be remedied by introducing adaptivity. Specifically, if we
select \emph{one} of the features with top attributions, e.g., feature $1$ and then
re-run integrated gradients with an updated baseline $(1, 0, \dots, 0)^\top$,
then the new attribution scores of the redundant features $2,\dots,t$ will be $0$. 
Therefore, the remaining $k-1$ top attributions will be placed on the remaining features
$t+1,\dots,n$.

\section{Experimental Evaluation}
In this section, we evaluate the effect of the adaptivity introduced by Greedy PIG.
A comparison with all attribution methods is out of scope,
mainly because our proposed modification can be easily adapted to
other algorithms (e.g., SmoothGrad), and so, even though we provide comparisons with different popular methods,
we mostly concentrate on integrated gradients.
We defer an evaluation of adaptive generalizations of other one-shot algorithms to future work.
\subsection{Feature attribution for image recognition}
\label{sec:feature_attribution}

\paragraph{Experimental setup.} We use the MobileNetV2 neural network~\citep{sandler2018mobilenetv2} pretrained on ImageNet, and
compare different feature attribution methods on individual Imagenet examples. We use the all-zero baseline for 
integrated gradients, which, because of the data 
preprocessing pipeline of MobileNetV2, corresponds to an
all-gray image. To ensure that our results hold across models, we present
additional results on ResNet50 in the appendix.

\paragraph{Top-class attribution.} \cite{kapishnikov2019xrai} introduced the softmax 
information curve for image classification, which plots the 
output softmax probability of the target class using only the top
attributions, as a function of the compressed image
size. Since we are interested in a
general method, we instead plot the output softmax as a function of the
number of top selected features.
To perform attribution based on our framework, 
we let $c^*$ be the class output by the model 
$f(\xx; \vtheta)$ on input $\xx$ and with parameters
$\vtheta$, and then define
$g^{\mathrm{top1}}(\mathbf{s}) := f_{c^*}(\xx \odot\mathbf{s}; \vtheta))$, 
where
$\mathbf{s} \in[0,1]^n$ and $\odot$ denotes the Hadamard product. In other words, 
$g^{\mathrm{top1}}$ gives the softmax output of the 
(fixed) class predicted by the model, after re-weighting
input features by $\mathbf{s}$. We can now run 
gradient-based attribution algorithms on $g$. The 
results can be found in Figure~\ref{fig:sic}
and
Figure~\ref{fig:tabby}.

\paragraph{Loss attribution.} Even though the top class attribution methodology can explain what are the most important
features that influence the model's top predicted class, it might fail to capture other
aspects of the model's behavior. For example, if the model has low certainty in its
prediction, the softmax output of the top class will be low. In fact, in such cases the 
attribution method can select a number of features that give a much more confident output.
However, it can instead be useful to capture the outputs of \emph{all classes}, and not just
the top one. More generally, we can use arbitrary loss functions to attribute certain 
aspects of the model's behavior to the features.

\begin{figure}[tb]
    \centering
    \begin{subfigure}[b]{\textwidth}
        \centering
        \includegraphics[width=0.35\linewidth]{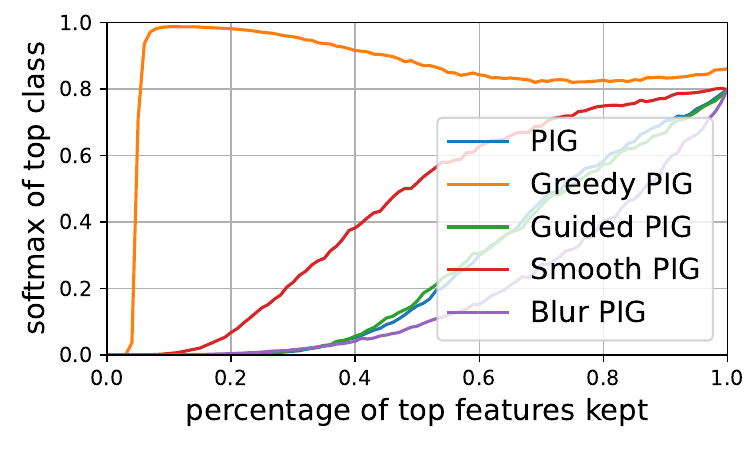}%
        \hspace{0.35cm}
        \includegraphics[width=0.35\linewidth]{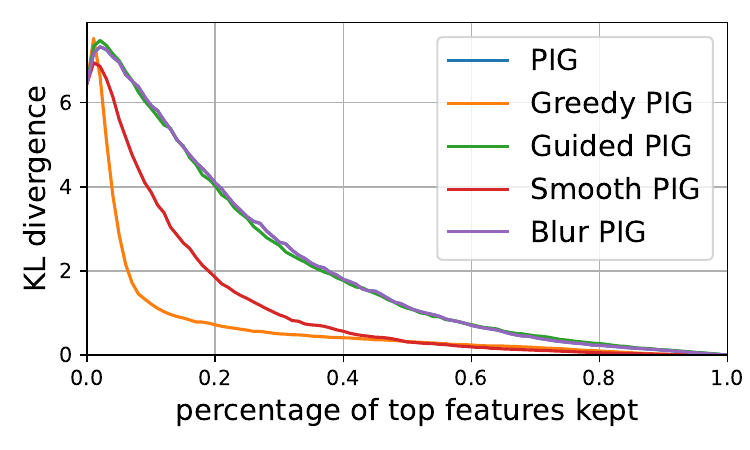}
    \end{subfigure}
    \begin{subfigure}[b]{\textwidth}
          \centering
          \begin{tabular}{lcc}
            \toprule
            \cmidrule(r){1-2}
            Algorithm     & softmax AUC $(\uparrow)$ & KL divergence AUC $(\downarrow)$ \\
            \midrule
            Integrated Gradients~\citep{sundararajan2017axiomatic} & $0.2639$ & $2.0812$ \\
            Smooth PIG~\citep{smilkov2017smoothgrad} & $0.4433$ & $1.1366$\\
            Blur PIG~\citep{xu2020attribution} & $0.1903$ & $2.0812$\\
            Guided PIG~\citep{kapishnikov2021guided} & $0.2623$ & $2.0644$\\
            Greedy PIG & ${\bf 0.8486}$ & ${\bf 0.6655}$ \\
            \bottomrule
          \end{tabular}
    \end{subfigure}
    \caption{Evaluating the attribution performance across $1466$ 
    examples randomly drawn from Imagenet. The Softmax
    Information Curve (SIC) plots the softmax output of the top class as
    a function of the number of features kept (there are
    $224\times 224\times 3$ features total for these RGB images).
    The Loss Information Curve
    (LIC) similarly plots the KL divergence from the output 
    probabilities on the features kept, to the output probabilities when
    all features are kept. To compute these results, we run 
    integrated gradients and guided integrated gradients
    with $2000$ steps and Greedy PIG with 
    $100$ rounds, $20$ steps each. Then, for each algorithm
    and number of features kept, we compute and plot the median
    across all examples, as in~\cite{kapishnikov2019xrai}. Left: The Softmax Information Curve (higher is better).
        Right: The Loss Information Curve (lower is better).
    }
    \label{fig:sic}
    \vspace{-0.2cm}
\end{figure}

\begin{figure}[tb]
    \centering
    \begin{subfigure}[b]{\textwidth}
        \centering
        \includegraphics[width=0.2\linewidth]{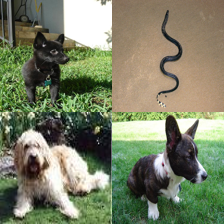}
        \includegraphics[width=0.2\linewidth]{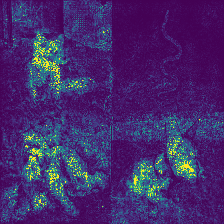}
        \includegraphics[width=0.2\linewidth]{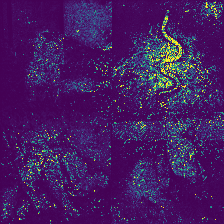}
        \caption{Top 15000 attributions for class ``sea snake''.
        Left: input, Middle: integrated gradients, Right: Greedy PIG.}
    \end{subfigure}
    \begin{subfigure}[b]{\textwidth}
          \centering
          \begin{tabular}{lc}
            \toprule
            \cmidrule(r){1-2}
            Algorithm     & pointing accuracy $(\uparrow)$ \\
            \midrule
            Integrated Gradients~\citep{sundararajan2017axiomatic} & $0.66$ \\
            Greedy PIG & ${\bf 0.83}$\\
            \bottomrule
          \end{tabular}
         \caption{Pointing accuracy aggregated over $25$ $2x2$ grids
         generated from examples that had the highest prediction confidence from $1500$ examples randomly drawn from the validation set of ImageNet.}
    \end{subfigure}
    \caption{Results from the pointing game~\cite{bohle2021convolutional}
    that is used for sanity checking image attribution methods. We 
    generate $2x2$ grids of the highest prediction confidence images, and
    obtain the attribution results for each class. For each (example, 
    class) pair, we count it as a positive if the majority of the top
    $15000$ attributions are on the quadrant associated with that class.
    We measure pointing accuracy as the fraction of such positives.
    }
    \label{fig:pointing1}
    \vspace{-0.2cm}
\end{figure}

Specifically, instead of asking which features are most responsible for a specific 
classification output, we ask: Which features are most responsible for the
\emph{distribution} on output classes? In other words,
how close is the model's multiclass output vector when fed a subset of features versus when fed all the features?
In order to answer this question, we use the loss function on which the model is trained.
For multiclass classification, this is usually set to using the negative cross entropy loss
$\ell(\oyy, \yy) = \langle \oyy, \log \yy\rangle$ to define a set function
$G(S) = \ell(f(\xx), f(\xx_S))$
and a corresponding continuous extension
$g^{\text{logloss}}(\mathbf{s})
= \ell(f(\xx), f(\xx\odot \mathbf{s}))$.
Maximizing $G$
is equivalent to maximizing $\ell(f(\xx), f(\xx_S))$. Fortunately, gradient-based attribution algorithms can be easily extended to arbitrary objective functions.
The results can be found in Figure~\ref{fig:sic} and 
Figure~\ref{fig:fish}.
\subsection{Edge-attribution to compress graphs and interpret GNNs}
\label{sec:graph_compression}

\begin{figure}[tb]
    \centering
    \begin{subfigure}[b]{0.27\textwidth}
        \includegraphics[height=2.2cm]{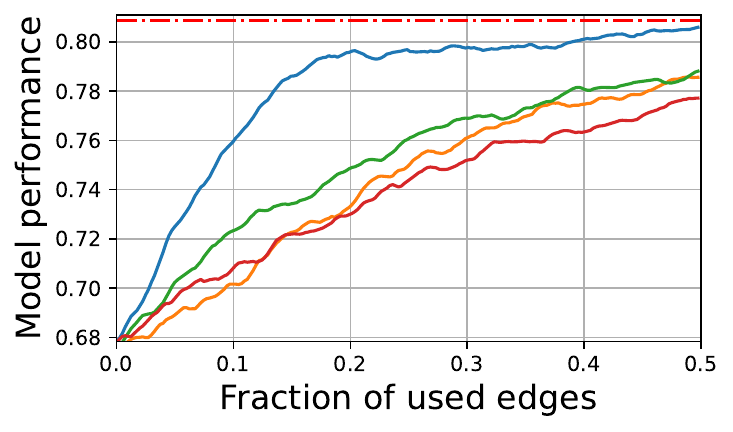}
        \caption{Cora}
    \end{subfigure}
    \begin{subfigure}[b]{0.27\textwidth}
        \includegraphics[height=2.2cm]{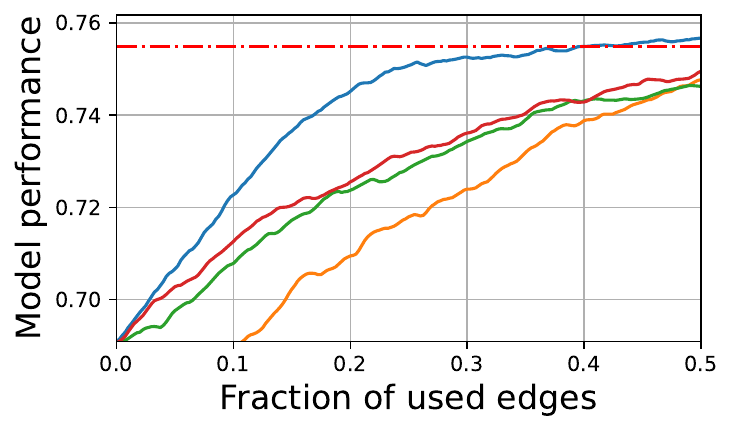}
        \caption{Pubmed}
    \end{subfigure}
    \begin{subfigure}[b]{0.27\textwidth}
        \includegraphics[height=2.2cm]{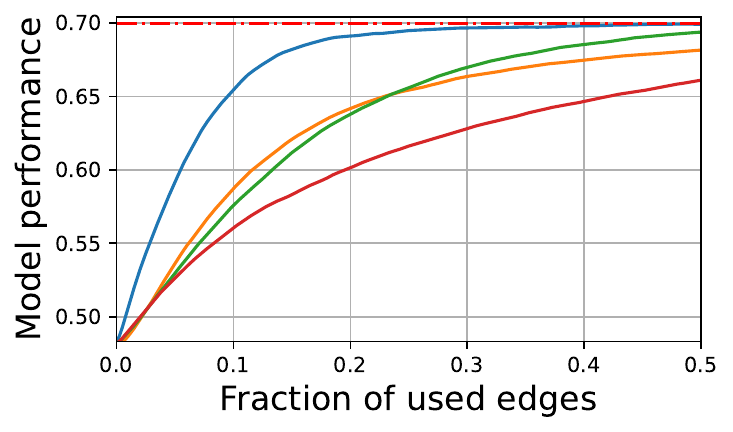}
        \caption{ogbn-arxiv}
    \end{subfigure}
    \begin{subfigure}[b]{0.15\textwidth}
        \raisebox{1.1cm}{\includegraphics[height=1.5cm]{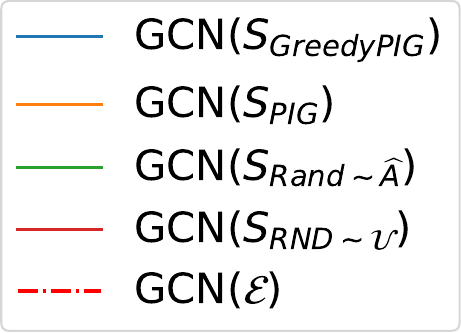}}
    \end{subfigure}
    \caption{Pre-trained accuracy GCN with inference using edge-subset selected by various methods.
    Solid lines plot accuracy vs.\ ratio of selected edges. Dashed line shows accuracy with all edges.}
    \label{fig:graph_compression}
\end{figure}

We use our method for \emph{graph compression} with pre-trained GCN of \citet{kipf}. We use a three-layer GCN, as: \ \ 
$
    \textrm{GCN}(\mathcal{E}; \theta) = \textbf{softmax}\left(\widehat{A_\mathcal{E}} \sigma\left(\widehat{A_\mathcal{E}} \sigma\left(\widehat{A_\mathcal{E}} \mathbf{X} \theta^{(1)}\right)\theta^{(2)}\right) \theta^{(3)}\right),
$
where
$\mathbf{X} \in \mathbb{R}^{m \times n}$ contains features of all $m$ nodes, 
$\mathcal{E} \subseteq [m] \times [m]$ denotes (undirected) edges with a corresponding sparse (symmetric) adjacency matrix with $(A_\mathcal{E})_{ij}=1$ iff  in $(i, j) \in \mathcal{E}$;
$\widehat{\cdot}$ is the symmetric-normalization with self-connections added, i.e.,
$\widehat{A} = (D + \mathbf{I}_m)^{-\frac12}(A+\mathbf{I}_m)(D + \mathbf{I}_m)^{-\frac12}$; diagonal degree matrix $D = \textrm{diag}(\mathbf{1}_m^\top A)$; and $\sigma$ denotes element-wise non-linearity such as ReLU.

Given GCN pre-trained for node classification \citep{kipf, mixhop, graphsage}, we are interested in inferring node labels using only a subset of the edges.
The edge subset can be chosen uniformly at random (Fig.~\ref{fig:graph_compression}: $S_{\textrm{RND} \sim \mathcal{U}})$;
or select edge $(i, j)$ with probability proportional to $ (D_{ii}D_{jj})^{-\frac12}$;
or according to edge-attribution assigned by PIG \citep{sundararajan2017axiomatic} ($S_{\textrm{PIG}})$;
or using our method Greedy PIG (Alg.~\ref{alg:greedy_pig}) ($S_{\textrm{GreedyPIG}}$), with
\begin{equation}
    g^\textrm{GNN}(\mathbf{s}; A^{(0)}, A) = A^{(0)} + \mathbf{s} \odot (A - A^{(0)} ) \ , \  \textrm{ \ with all-zero \ } A^{(0)} = \mathbf{0} \ , \ \textrm{ \ and sparse \ } \mathbf{s} \in \mathbb{R}^{m \times m}.
\end{equation}

Figure~\ref{fig:graph_compression} shows that Greedy PIG can better compress the graph while maintaining performance of the underlying model. The GCN model was pre-trained on the full graph with TF-GNN \citep{tfgnn}, on datasets of \textbf{ogbn-arxiv} from \citet{ogb}, and \textbf{Cora} and \textbf{Pubmed} from \citep{planetoid}.
Inference on the full graph (using all edges) gives performance indicated with a dashed-red line. 
We observe that removing $>50\%$ of the graph edges, using GreedyPIG, has negligible effect on GCN performance.
For the GCN model assumptions, we ensure the selected edges $S\subseteq \mathcal{E}$ correspond to symmetric $A_{S} = A_{S}^\top$ by projecting the output of the gradient oracle 
on to the space of symmetric matrices---by averaging the (sparse) gradient with its transpose, \textit{i.e.}, by re-defining:
\begin{equation}
    \widetilde{\nabla_\mathbf{s} g(\mathbf{s})} \overset{\triangle}{=} \frac12 \frac{\partial}{\partial \mathbf{s}} g^\textrm{GNN}(\mathbf{s}) + \frac12 \left(\frac{\partial}{\partial \mathbf{s}} g^\textrm{GNN}(\mathbf{s}) \right)^\top \ \ \ \textrm{for use in Alg.~\ref{alg:greedy_pig}}
\end{equation}

We also use Greedy PIG to interpret the GCN. Specifically,
zooming-into any particular graph node, we would like to to explain the subgraph around the node that leads to the GCN output.
Due to space constraints, we report in the appendix subgraphs around nodes in the ogbn-arxiv dataset. Note that \citet{sanchez2020attribute_gnn}
apply integrated gradients
on smaller graphs (e.g., chemical molecules with at most 100 edges),
whereas we consider large graphs (e.g., millions of edges).
\subsection{Post-hoc feature selection on Criteo tabular CTR dataset}
\label{sec:feature_selection}

Feature selection is the task of selecting a subset of features on which to train a model, and dropping the remaining features. There are many approaches
for selecting the most predictive set of features~\citep{li2017feature}. Even though
in general feature selection and model training can be intertwined, often
it is easier to decouple these processes. Specifically, given a trained
model, it is often desirable to perform \emph{post-hoc} feature selection,
where one only has the ability to inspect and not modify the model. This is also related to global feature attribution~\citep{zhang2021survey}.

\paragraph{Methodology.} In Table~\ref{table:criteo_results}, we compare the quality of different
attribution methods for post-hoc feature selection. To evaluate the quality
of feature selection, we employ the following approach: {\bf 1.} Train a neural network model using all features. {\bf 2.} Run each attribution method to compute \emph{global} attribution scores which are aggregate attribution scores over many input examples. {\bf 3.} Pick the top-$k$ attributed features, and train new pruned models that only use these top features. Report the validation loss of the resulting pruned models for different values of $k$. 
The results of the methodology described above applied on a 
random subset of the Criteo
dataset are presented in Table~\ref{table:criteo_results}.
We follow the setup of~\cite{yasuda2023sequential}, who define a dense neural network with
three hidden layers with sizes $768$, $256$ and $128$.
To make the comparison fair, we guarantee that all algorithms consume the same amount
of data (gradient batches). For example, Greedy PIG with $T=1$ steps (which we also
call Sequential Gradient, see~Section~\ref{sec:greedy_pig}),
uses 5x more data
batches for each gradient computation than Greedy PIG with $T=5$.

\begin{table}[t]
  \caption{Cross-entropy loss on Criteo CTR dataset using only the top-$k$ features. 
  Reported is the minimum validation loss after 20K training steps with batch size 512.
  During selection, all of the algorithms consume 
  the same number of batches.}
  \centering
  \begin{tabular}{lcccc}
    \toprule
    \cmidrule(r){1-2}
    Algorithm     & $k=5$     & $k=10$ & $k=20$ & $k=30$ \\
    \midrule
    Integrated Gradients ($T=39$) & $0.4827$ & $0.4641$ & $0.4605$ & $0.4533$ \\
    \midrule
    Greedy PIG ($T=1$) & $0.4728$ & $0.4629$ & $0.4551$ &  ${\bf 0.4508}$\\
    Greedy PIG ($T=5$) & ${\bf 0.4723}$ & ${\bf 0.4627}$ & ${\bf 0.4531}$ & $0.4509$\\ 
    \bottomrule
  \end{tabular}
  \label{table:criteo_results}
\end{table}

\section{Related Work}

Feature attribution, also known as computing the \emph{saliency} of each feature,
consists of a large class of methods for explaining the predictions of neural networks.
The attribution score describes the impact of each feature on the neural network's output.
In the following, we start with most-related to our research, and then move onto broader and further topics.

\paragraph{Integrated gradients.}
Our work is most related to (and generalizes to set functions) the work of ~\citet{sundararajan2017axiomatic,kapishnikov2021guided,lundstrom2022rigorous,qi2019visualizing,goh2021understanding,sattarzadeh2021integrated}.
In their seminal work,
\citet{sundararajan2017axiomatic}
proposed
path integrated gradients (PIG), 
a method that assigns an importance score to each feature by integrating the partial derivative with respect to a feature along a path that interpolates between the background ``\textit{baseline}'' and the input features.
PIG, however, has some limitations. For example, it can be sensitive to the choice of baseline input. Further, it can be difficult to interpret the results of path integrated gradients when there are duplicate features or features that carry common information. 
It has also been observed that PIG is
sensitive to input noise~\citep{smilkov2017smoothgrad},
model re-trains~\citep{hooker2019benchmark}, and
in some cases
provably unable to identify spurious 
features~\citep{bilodeau2022impossibility}.

\cite{simonyan2014deep} used gradient ascent to explore the internal workings of neural networks
and to quantify the effect of each individual pixel on the classification prediction.
\cite{springenberg2015striving} presented a neural network that is composed entirely of convolutional layers and introduced guided backpropagation, a technique for eliminating negative signals generated during backpropagation when performing standard gradient ascent.
Grad-CAM \citep{selvaraju2017grad} produces a saliency map by using the gradients of any target concept flowing into the final convolutional layer. SmoothGrad~\citep{smilkov2017smoothgrad} starts by computing the gradient of the class score function with respect to the input image.
However, it visually sharpens the gradient-based sensitivity maps by adding noise to the input image and computing the gradient for each of the perturbed images. Averaging the sensitivity maps together produces a clearer result. For a comprehensive survey of feature attribution and neural network interpretability, see \citet{zhang2021survey}.

\paragraph{Shapley value.} Shapley value was originally introduced in game theory~\citep{shapley1953value}. While it can be directly applied to explain the predictions of neural networks, it requires extremely high computational complexity. \cite{lundberg2017unified} proposed the SHapley Additive exPlanations (SHAP) algorithm, which approximates the Shapley value of features.  \citet{sundararajan2020many} employed an axiomatic approach to investigate the differences between various versions of the Shapley value for attribution, and they also discussed a technique called Baseline Shapley (BShap). The following are other Shapley value-based methods~\citep{chenshapley,fryeshapley}.

\paragraph{Gradients and backpropgation.} Another class of interpretability methods is based on gradients and backpropagation.
These methods compute the gradient of each feature, which is used as a measure of the feature's importance~\citep{baehrens2010explain,simonyan2014deep}. Integrated gradient is an instance of the gradient-based method. Backpropagation-based methods design backpropagation rules for convolutional, pooling, and nonlinear activation layers so  they can assign importance scores in a fair and reasonable manner during backpropagation~\citep{shrikumar2017learning}. The guided backpropagation (GBP) algorithm~\citep{springenberg2015striving}, the layer-wise relevance propagation (LRP) algorithm~\citep{bach2015pixel}, and integrated gradient (IG) algorithm~\citep{sundararajan2017axiomatic} are three notable examples of this class of methods.

\paragraph{Model-agnostic explanation.}
\citet{ribeiro2016model} proposed Local Interpretable Model-Agnostic Explanations (LIME), a method that uses a trained local proxy model to provide explanation results without the need for backpropagation to compute gradients. \citet{plumb2018model} used a similar method that relies on a local linear model for explanation.

\section{Conclusion}
We view feature attribution and feature selection as instances of subset selection, which
allows us to apply well-established theories and approximations from the field of discrete and submodular optimization.
Through these,
we give a greedy approximation using path integrated gradients (PIG), which we coin Greedy PIG.
We show that Greedy PIG can succeed in scenarios where other integrated gradient methods fail, e.g., when features are correlated.
We qualitatively show the efficacy of Greedy PIG---it
explains predictions of ImageNet-trained model as attributions on input images,
and it explains Graph Neural Network (GNN) as attributions on edges on ogbn-arxiv graph.
We quantitatively evaluate Greedy PIG for feature selection on the Criteo CTR problem,
and by compressing a graph by preserving only a subset of its edges while maintaining GNN performance.
Our evaluations show that Greedy PIG gives qualitatively better model interpretations than alternatives,
and scores higher on quantitative evaluation metrics.

\bibliography{references}

\begin{thebibliography}{62}
\providecommand{\natexlab}[1]{#1}
\providecommand{\url}[1]{\texttt{#1}}
\expandafter\ifx\csname urlstyle\endcsname\relax
  \providecommand{\doi}[1]{doi: #1}\else
  \providecommand{\doi}{doi: \begingroup \urlstyle{rm}\Url}\fi

\bibitem[Abebe et~al.(2022)Abebe, Hardt, Jin, Miller, Schmidt, and
  Wexler]{abebe2022adversarial}
Rediet Abebe, Moritz Hardt, Angela Jin, John Miller, Ludwig Schmidt, and
  Rebecca Wexler.
\newblock Adversarial scrutiny of evidentiary statistical software.
\newblock In \emph{2022 ACM Conference on Fairness, Accountability, and
  Transparency}, FAccT '22, page 1733–1746, 2022.

\bibitem[Abu{-}El{-}Haija et~al.(2019)Abu{-}El{-}Haija, Perozzi, Kapoor,
  Harutyunyan, Alipourfard, Lerman, Steeg, and Galstyan]{mixhop}
Sami Abu{-}El{-}Haija, Bryan Perozzi, Amol Kapoor, Hrayr Harutyunyan, Nazanin
  Alipourfard, Kristina Lerman, Greg~Ver Steeg, and Aram Galstyan.
\newblock Mixhop: Higher-order graph convolutional architectures via sparsified
  neighborhood mixing.
\newblock In \emph{International Conference on Machine Learning}, ICML, 2019.

\bibitem[Adebayo et~al.(2018)Adebayo, Gilmer, Muelly, Goodfellow, Hardt, and
  Kim]{adebayo2018sanity}
Julius Adebayo, Justin Gilmer, Michael Muelly, Ian Goodfellow, Moritz Hardt,
  and Been Kim.
\newblock Sanity checks for saliency maps.
\newblock \emph{Advances in neural information processing systems}, 31, 2018.

\bibitem[Ahmad et~al.(2018)Ahmad, Eckert, and
  Teredesai]{ahmad2018interpretable}
Muhammad~Aurangzeb Ahmad, Carly Eckert, and Ankur Teredesai.
\newblock Interpretable machine learning in healthcare.
\newblock In \emph{Proceedings of the 2018 ACM International Conference on
  Bioinformatics, Computational Biology, and Health Informatics}, pages
  559--560, 2018.

\bibitem[Ancona et~al.(2017)Ancona, Ceolini, {\"O}ztireli, and
  Gross]{ancona2017towards}
Marco Ancona, Enea Ceolini, Cengiz {\"O}ztireli, and Markus Gross.
\newblock Towards better understanding of gradient-based attribution methods
  for deep neural networks.
\newblock \emph{arXiv preprint arXiv:1711.06104}, 2017.

\bibitem[Axiotis and Yasuda(2023)]{axiotis2023performance}
Kyriakos Axiotis and Taisuke Yasuda.
\newblock Performance of $\ell_1$ regularization for sparse convex
  optimization.
\newblock \emph{arXiv preprint arXiv:2307.07405}, 2023.

\bibitem[Bach et~al.(2015)Bach, Binder, Montavon, Klauschen, M{\"u}ller, and
  Samek]{bach2015pixel}
Sebastian Bach, Alexander Binder, Gr{\'e}goire Montavon, Frederick Klauschen,
  Klaus-Robert M{\"u}ller, and Wojciech Samek.
\newblock On pixel-wise explanations for non-linear classifier decisions by
  layer-wise relevance propagation.
\newblock \emph{PloS one}, 10\penalty0 (7):\penalty0 e0130140, 2015.

\bibitem[Baehrens et~al.(2010)Baehrens, Schroeter, Harmeling, Kawanabe, Hansen,
  and M{\"u}ller]{baehrens2010explain}
David Baehrens, Timon Schroeter, Stefan Harmeling, Motoaki Kawanabe, Katja
  Hansen, and Klaus-Robert M{\"u}ller.
\newblock How to explain individual classification decisions.
\newblock \emph{The Journal of Machine Learning Research}, 11:\penalty0
  1803--1831, 2010.

\bibitem[Balkanski and Singer(2018)]{balkanski2018adaptive}
Eric Balkanski and Yaron Singer.
\newblock The adaptive complexity of maximizing a submodular function.
\newblock In \emph{Proceedings of the 50th Annual ACM SIGACT Symposium on
  Theory of Computing}, pages 1138--1151, 2018.

\bibitem[Bateni et~al.(2019)Bateni, Chen, Esfandiari, Fu, Mirrokni, and
  Rostamizadeh]{bateni2019categorical}
MohammadHossein Bateni, Lin Chen, Hossein Esfandiari, Thomas Fu, Vahab
  Mirrokni, and Afshin Rostamizadeh.
\newblock Categorical feature compression via submodular optimization.
\newblock In \emph{International Conference on Machine Learning}, pages
  515--523. PMLR, 2019.

\bibitem[Bilodeau et~al.(2022)Bilodeau, Jaques, Koh, and
  Kim]{bilodeau2022impossibility}
Blair Bilodeau, Natasha Jaques, Pang~Wei Koh, and Been Kim.
\newblock Impossibility theorems for feature attribution.
\newblock \emph{arXiv preprint arXiv:2212.11870}, 2022.

\bibitem[Bohle et~al.(2021)Bohle, Fritz, and Schiele]{bohle2021convolutional}
Moritz Bohle, Mario Fritz, and Bernt Schiele.
\newblock Convolutional dynamic alignment networks for interpretable
  classifications.
\newblock In \emph{Proceedings of the IEEE/CVF Conference on Computer Vision
  and Pattern Recognition}, pages 10029--10038, 2021.

\bibitem[B{\"o}hle et~al.(2022)B{\"o}hle, Fritz, and Schiele]{bohle2022b}
Moritz B{\"o}hle, Mario Fritz, and Bernt Schiele.
\newblock B-cos networks: Alignment is all we need for interpretability.
\newblock In \emph{Proceedings of the IEEE/CVF Conference on Computer Vision
  and Pattern Recognition}, pages 10329--10338, 2022.

\bibitem[Bubeck et~al.(2023)Bubeck, Chandrasekaran, Eldan, Gehrke, Horvitz,
  Kamar, Lee, Lee, Li, Lundberg, et~al.]{bubeck2023sparks}
S{\'e}bastien Bubeck, Varun Chandrasekaran, Ronen Eldan, Johannes Gehrke, Eric
  Horvitz, Ece Kamar, Peter Lee, Yin~Tat Lee, Yuanzhi Li, Scott Lundberg,
  et~al.
\newblock Sparks of artificial general intelligence: Early experiments with
  {GPT}-4.
\newblock \emph{arXiv preprint arXiv:2303.12712}, 2023.

\bibitem[Cai et~al.(2018)Cai, Luo, Wang, and Yang]{cai2018feature}
Jie Cai, Jiawei Luo, Shulin Wang, and Sheng Yang.
\newblock Feature selection in machine learning: A new perspective.
\newblock \emph{Neurocomputing}, 300:\penalty0 70--79, 2018.

\bibitem[Chen et~al.(2021{\natexlab{a}})Chen, Pierson, Rose, Joshi, Ferryman,
  and Ghassemi]{chen2021ethical}
Irene~Y Chen, Emma Pierson, Sherri Rose, Shalmali Joshi, Kadija Ferryman, and
  Marzyeh Ghassemi.
\newblock Ethical machine learning in healthcare.
\newblock \emph{Annual Review of Biomedical Data Science}, 4:\penalty0
  123--144, 2021{\natexlab{a}}.

\bibitem[Chen et~al.(2019)Chen, Song, Wainwright, and Jordan]{chenshapley}
Jianbo Chen, Le~Song, Martin~J Wainwright, and Michael~I Jordan.
\newblock L-shapley and c-shapley: Efficient model interpretation for
  structured data.
\newblock In \emph{International Conference on Learning Representations}, 2019.

\bibitem[Chen et~al.(2021{\natexlab{b}})Chen, Esfandiari, Fu, Mirrokni, and
  Yu]{chen2021feature}
Lin Chen, Hossein Esfandiari, Gang Fu, Vahab~S Mirrokni, and Qian Yu.
\newblock Feature cross search via submodular optimization.
\newblock In \emph{29th Annual European Symposium on Algorithms (ESA 2021)}.
  Schloss Dagstuhl-Leibniz-Zentrum f{\"u}r Informatik, 2021{\natexlab{b}}.

\bibitem[Elenberg et~al.(2018)Elenberg, Khanna, Dimakis, and
  Negahban]{elenberg2018restricted}
Ethan~R Elenberg, Rajiv Khanna, Alexandros~G Dimakis, and Sahand Negahban.
\newblock Restricted strong convexity implies weak submodularity.
\newblock \emph{The Annals of Statistics}, 46\penalty0 (6B):\penalty0
  3539--3568, 2018.

\bibitem[Fahrbach et~al.(2019{\natexlab{a}})Fahrbach, Mirrokni, and
  Zadimoghaddam]{fahrbach2019non}
Matthew Fahrbach, Vahab Mirrokni, and Morteza Zadimoghaddam.
\newblock Non-monotone submodular maximization with nearly optimal adaptivity
  and query complexity.
\newblock In \emph{International Conference on Machine Learning}, pages
  1833--1842. PMLR, 2019{\natexlab{a}}.

\bibitem[Fahrbach et~al.(2019{\natexlab{b}})Fahrbach, Mirrokni, and
  Zadimoghaddam]{fahrbach2019submodular}
Matthew Fahrbach, Vahab Mirrokni, and Morteza Zadimoghaddam.
\newblock Submodular maximization with nearly optimal approximation, adaptivity
  and query complexity.
\newblock In \emph{Proceedings of the Thirtieth Annual ACM-SIAM Symposium on
  Discrete Algorithms}, pages 255--273. SIAM, 2019{\natexlab{b}}.

\bibitem[Ferludin et~al.(2022)Ferludin, Eigenwillig, Blais, Zelle, Pfeifer,
  Sanchez-Gonzalez, Li, Abu-El-Haija, Battaglia, Bulut, et~al.]{tfgnn}
Oleksandr Ferludin, Arno Eigenwillig, Martin Blais, Dustin Zelle, Jan Pfeifer,
  Alvaro Sanchez-Gonzalez, Sibon Li, Sami Abu-El-Haija, Peter Battaglia,
  Neslihan Bulut, et~al.
\newblock {TF-GNN}: Graph neural networks in tensorflow.
\newblock \emph{arXiv preprint arXiv:2207.03522}, 2022.

\bibitem[Frye et~al.(2021)Frye, de~Mijolla, Begley, Cowton, Stanley, and
  Feige]{fryeshapley}
Christopher Frye, Damien de~Mijolla, Tom Begley, Laurence Cowton, Megan
  Stanley, and Ilya Feige.
\newblock Shapley explainability on the data manifold.
\newblock In \emph{International Conference on Learning Representations}, 2021.

\bibitem[Goh et~al.(2021)Goh, Lapuschkin, Weber, Samek, and
  Binder]{goh2021understanding}
Gary~SW Goh, Sebastian Lapuschkin, Leander Weber, Wojciech Samek, and Alexander
  Binder.
\newblock Understanding integrated gradients with smoothtaylor for deep neural
  network attribution.
\newblock In \emph{2020 25th International Conference on Pattern Recognition
  (ICPR)}, pages 4949--4956. IEEE, 2021.

\bibitem[Graves et~al.(2013)Graves, Mohamed, and Hinton]{graves2013speech}
Alex Graves, Abdel-rahman Mohamed, and Geoffrey Hinton.
\newblock Speech recognition with deep recurrent neural networks.
\newblock In \emph{2013 IEEE international conference on acoustics, speech and
  signal processing}, pages 6645--6649. Ieee, 2013.

\bibitem[Hamilton et~al.(2017)Hamilton, Ying, and Leskovec]{graphsage}
W.~Hamilton, R.~Ying, and J.~Leskovec.
\newblock Inductive representation learning on large graphs.
\newblock In \emph{Advances in Neural Information Processing Systems}, 2017.

\bibitem[He et~al.(2016)He, Zhang, Ren, and Sun]{he2016deep}
Kaiming He, Xiangyu Zhang, Shaoqing Ren, and Jian Sun.
\newblock Deep residual learning for image recognition.
\newblock In \emph{Proceedings of the IEEE conference on computer vision and
  pattern recognition}, pages 770--778, 2016.

\bibitem[Hooker et~al.(2019)Hooker, Erhan, Kindermans, and
  Kim]{hooker2019benchmark}
Sara Hooker, Dumitru Erhan, Pieter-Jan Kindermans, and Been Kim.
\newblock A benchmark for interpretability methods in deep neural networks.
\newblock \emph{Advances in neural information processing systems}, 32, 2019.

\bibitem[Hu et~al.(2020)Hu, Fey, Zitnik, Dong, Ren, Liu, Catasta, and
  Leskovec]{ogb}
Weihua Hu, Matthias Fey, Marinka Zitnik, Yuxiao Dong, Hongyu Ren, Bowen Liu,
  Michele Catasta, and Jure Leskovec.
\newblock Open graph benchmark: Datasets for machine learning on graphs.
\newblock In \emph{arXiv}, 2020.

\bibitem[Kapishnikov et~al.(2019)Kapishnikov, Bolukbasi, Vi{\'e}gas, and
  Terry]{kapishnikov2019xrai}
Andrei Kapishnikov, Tolga Bolukbasi, Fernanda Vi{\'e}gas, and Michael Terry.
\newblock Xrai: Better attributions through regions.
\newblock In \emph{Proceedings of the IEEE/CVF International Conference on
  Computer Vision}, pages 4948--4957, 2019.

\bibitem[Kapishnikov et~al.(2021)Kapishnikov, Venugopalan, Avci, Wedin, Terry,
  and Bolukbasi]{kapishnikov2021guided}
Andrei Kapishnikov, Subhashini Venugopalan, Besim Avci, Ben Wedin, Michael
  Terry, and Tolga Bolukbasi.
\newblock Guided integrated gradients: An adaptive path method for removing
  noise.
\newblock In \emph{Proceedings of the IEEE/CVF Conference on Computer Vision
  and Pattern Recognition}, pages 5050--5058, 2021.

\bibitem[Kipf and Welling(2017)]{kipf}
Thomas Kipf and Max Welling.
\newblock Semi-supervised classification with graph convolutional networks.
\newblock In \emph{International Conference on Learning Representations}, 2017.

\bibitem[Kleindessner et~al.(2019)Kleindessner, Awasthi, and
  Morgenstern]{kleindessner2019fair}
Matth{\"a}us Kleindessner, Pranjal Awasthi, and Jamie Morgenstern.
\newblock Fair k-center clustering for data summarization.
\newblock In \emph{International Conference on Machine Learning}, pages
  3448--3457. PMLR, 2019.

\bibitem[Krause and Golovin(2014)]{krause2014submodular}
Andreas Krause and Daniel Golovin.
\newblock Submodular function maximization.
\newblock \emph{Tractability}, 3:\penalty0 71--104, 2014.

\bibitem[Krizhevsky et~al.(2017)Krizhevsky, Sutskever, and
  Hinton]{krizhevsky2017imagenet}
Alex Krizhevsky, Ilya Sutskever, and Geoffrey~E Hinton.
\newblock Imagenet classification with deep convolutional neural networks.
\newblock \emph{Communications of the ACM}, 60\penalty0 (6):\penalty0 84--90,
  2017.

\bibitem[Li et~al.(2017)Li, Cheng, Wang, Morstatter, Trevino, Tang, and
  Liu]{li2017feature}
Jundong Li, Kewei Cheng, Suhang Wang, Fred Morstatter, Robert~P Trevino,
  Jiliang Tang, and Huan Liu.
\newblock Feature selection: A data perspective.
\newblock \emph{ACM computing surveys (CSUR)}, 50\penalty0 (6):\penalty0 1--45,
  2017.

\bibitem[Liberty and Sviridenko(2017)]{liberty2017greedy}
Edo Liberty and Maxim Sviridenko.
\newblock Greedy minimization of weakly supermodular set functions.
\newblock In \emph{Approximation, Randomization, and Combinatorial
  Optimization. Algorithms and Techniques (APPROX/RANDOM 2017)}. Schloss
  Dagstuhl-Leibniz-Zentrum fuer Informatik, 2017.

\bibitem[Liu et~al.(2023)Liu, Wang, Britton, and Abebe]{abebe2023students}
Lydia~T. Liu, Serena Wang, Tolani Britton, and Rediet Abebe.
\newblock Reimagining the machine learning life cycle to improve educational
  outcomes of students.
\newblock \emph{Proceedings of the National Academy of Sciences}, 120\penalty0
  (9):\penalty0 e2204781120, 2023.

\bibitem[Lundberg and Lee(2017)]{lundberg2017unified}
Scott~M Lundberg and Su-In Lee.
\newblock A unified approach to interpreting model predictions.
\newblock \emph{Advances in neural information processing systems}, 30, 2017.

\bibitem[Lundstrom et~al.(2022)Lundstrom, Huang, and
  Razaviyayn]{lundstrom2022rigorous}
Daniel~D Lundstrom, Tianjian Huang, and Meisam Razaviyayn.
\newblock A rigorous study of integrated gradients method and extensions to
  internal neuron attributions.
\newblock In \emph{International Conference on Machine Learning}, pages
  14485--14508. PMLR, 2022.

\bibitem[Mirrokni and Zadimoghaddam(2015)]{mirrokni2015randomized}
Vahab Mirrokni and Morteza Zadimoghaddam.
\newblock Randomized composable core-sets for distributed submodular
  maximization.
\newblock In \emph{Proceedings of the forty-seventh annual ACM Symposium on
  Theory of Computing}, pages 153--162, 2015.

\bibitem[Nemhauser et~al.(1978)Nemhauser, Wolsey, and
  Fisher]{nemhauser1978analysis}
George~L Nemhauser, Laurence~A Wolsey, and Marshall~L Fisher.
\newblock An analysis of approximations for maximizing submodular set
  functions—i.
\newblock \emph{Mathematical programming}, 14:\penalty0 265--294, 1978.

\bibitem[Plumb et~al.(2018)Plumb, Molitor, and Talwalkar]{plumb2018model}
Gregory Plumb, Denali Molitor, and Ameet~S Talwalkar.
\newblock Model agnostic supervised local explanations.
\newblock \emph{Advances in neural information processing systems}, 31, 2018.

\bibitem[Qi et~al.(2019)Qi, Khorram, and Li]{qi2019visualizing}
Zhongang Qi, Saeed Khorram, and Fuxin Li.
\newblock Visualizing deep networks by optimizing with integrated gradients.
\newblock In \emph{CVPR Workshops}, volume~2, pages 1--4, 2019.

\bibitem[Ribeiro et~al.(2016)Ribeiro, Singh, and Guestrin]{ribeiro2016model}
Marco~Tulio Ribeiro, Sameer~Singh Singh, and Carlos Guestrin.
\newblock Model-agnostic interpretability of machine learning.
\newblock In \emph{Proceedings of the 2016 ICML Workshop on Human
  Interpretability in Machine Learning (WHI 2016)}. New York, NY, 2016.

\bibitem[Saliency(2017)]{saliency}
Saliency.
\newblock Framework-agnostic implementation for state-of-the-art saliency
  methods (xrai, blurig, smoothgrad, and more).
\newblock \url{https://github.com/PAIR-code/saliency}, 2017.

\bibitem[Sanchez-Lengeling et~al.(2020)Sanchez-Lengeling, Wei, Lee, Reif, Wang,
  Qian, McCloskey, Colwell, and Wiltschko]{sanchez2020attribute_gnn}
Benjamin Sanchez-Lengeling, Jennifer Wei, Brian Lee, Emily Reif, Peter Wang,
  Wesley Qian, Kevin McCloskey, Lucy Colwell, and Alexander Wiltschko.
\newblock Evaluating attribution for graph neural networks.
\newblock In \emph{Advances in Neural Information Processing Systems},
  volume~33, pages 5898--5910, 2020.

\bibitem[Sandler et~al.(2018)Sandler, Howard, Zhu, Zhmoginov, and
  Chen]{sandler2018mobilenetv2}
Mark Sandler, Andrew Howard, Menglong Zhu, Andrey Zhmoginov, and Liang-Chieh
  Chen.
\newblock Mobilenetv2: Inverted residuals and linear bottlenecks.
\newblock In \emph{Proceedings of the IEEE conference on computer vision and
  pattern recognition}, pages 4510--4520, 2018.

\bibitem[Sattarzadeh et~al.(2021)Sattarzadeh, Sudhakar, Plataniotis, Jang,
  Jeong, and Kim]{sattarzadeh2021integrated}
Sam Sattarzadeh, Mahesh Sudhakar, Konstantinos~N Plataniotis, Jongseong Jang,
  Yeonjeong Jeong, and Hyunwoo Kim.
\newblock Integrated grad-cam: Sensitivity-aware visual explanation of deep
  convolutional networks via integrated gradient-based scoring.
\newblock In \emph{ICASSP 2021-2021 IEEE International Conference on Acoustics,
  Speech and Signal Processing (ICASSP)}, pages 1775--1779. IEEE, 2021.

\bibitem[Selvaraju et~al.(2017)Selvaraju, Cogswell, Das, Vedantam, Parikh, and
  Batra]{selvaraju2017grad}
Ramprasaath~R Selvaraju, Michael Cogswell, Abhishek Das, Ramakrishna Vedantam,
  Devi Parikh, and Dhruv Batra.
\newblock Grad-cam: Visual explanations from deep networks via gradient-based
  localization.
\newblock In \emph{Proceedings of the IEEE international conference on computer
  vision}, pages 618--626, 2017.

\bibitem[Shapley(1953)]{shapley1953value}
Lloyd~S Shapley.
\newblock A value for n-person games.
\newblock \emph{Contributions to Theory Games (AM-28)}, 1953.

\bibitem[Shrikumar et~al.(2017)Shrikumar, Greenside, and
  Kundaje]{shrikumar2017learning}
Avanti Shrikumar, Peyton Greenside, and Anshul Kundaje.
\newblock Learning important features through propagating activation
  differences.
\newblock In \emph{International conference on machine learning}, pages
  3145--3153. PMLR, 2017.

\bibitem[Simonyan et~al.(2014)Simonyan, Vedaldi, and
  Zisserman]{simonyan2014deep}
K~Simonyan, A~Vedaldi, and A~Zisserman.
\newblock Deep inside convolutional networks: visualising image classification
  models and saliency maps.
\newblock In \emph{Proceedings of the International Conference on Learning
  Representations (ICLR)}. ICLR, 2014.

\bibitem[Smilkov et~al.(2017)Smilkov, Thorat, Kim, Vi{\'e}gas, and
  Wattenberg]{smilkov2017smoothgrad}
Daniel Smilkov, Nikhil Thorat, Been Kim, Fernanda Vi{\'e}gas, and Martin
  Wattenberg.
\newblock Smoothgrad: removing noise by adding noise.
\newblock \emph{arXiv preprint arXiv:1706.03825}, 2017.

\bibitem[Springenberg et~al.(2015)Springenberg, Dosovitskiy, Brox, and
  Riedmiller]{springenberg2015striving}
J~Springenberg, Alexey Dosovitskiy, Thomas Brox, and M~Riedmiller.
\newblock Striving for simplicity: The all convolutional net.
\newblock In \emph{ICLR (workshop track)}, 2015.

\bibitem[Sundararajan and Najmi(2020)]{sundararajan2020many}
Mukund Sundararajan and Amir Najmi.
\newblock The many shapley values for model explanation.
\newblock In \emph{International conference on machine learning}, pages
  9269--9278. PMLR, 2020.

\bibitem[Sundararajan et~al.(2017)Sundararajan, Taly, and
  Yan]{sundararajan2017axiomatic}
Mukund Sundararajan, Ankur Taly, and Qiqi Yan.
\newblock Axiomatic attribution for deep networks.
\newblock In \emph{International conference on machine learning}, pages
  3319--3328. PMLR, 2017.

\bibitem[Wiens and Shenoy(2018)]{wiens2018machine}
Jenna Wiens and Erica~S Shenoy.
\newblock Machine learning for healthcare: on the verge of a major shift in
  healthcare epidemiology.
\newblock \emph{Clinical Infectious Diseases}, 66\penalty0 (1):\penalty0
  149--153, 2018.

\bibitem[Xu et~al.(2020)Xu, Venugopalan, and Sundararajan]{xu2020attribution}
Shawn Xu, Subhashini Venugopalan, and Mukund Sundararajan.
\newblock Attribution in scale and space.
\newblock In \emph{Proceedings of the IEEE/CVF Conference on Computer Vision
  and Pattern Recognition}, pages 9680--9689, 2020.

\bibitem[Yang et~al.(2016)Yang, Cohen, and Salakhutdinov]{planetoid}
Zhilin Yang, William~W Cohen, and Ruslan Salakhutdinov.
\newblock Revisiting semi-supervised learning with graph embeddings.
\newblock In \emph{International Conference on Machine Learning}, 2016.

\bibitem[Yasuda et~al.(2023)Yasuda, Bateni, Chen, Fahrbach, Fu, and
  Mirrokni]{yasuda2023sequential}
Taisuke Yasuda, MohammadHossein Bateni, Lin Chen, Matthew Fahrbach, Gang Fu,
  and Vahab Mirrokni.
\newblock Sequential attention for feature selection.
\newblock \emph{ICLR 2023}, 2023.

\bibitem[Zhang et~al.(2021)Zhang, Ti{\v{n}}o, Leonardis, and
  Tang]{zhang2021survey}
Yu~Zhang, Peter Ti{\v{n}}o, Ale{\v{s}} Leonardis, and Ke~Tang.
\newblock A survey on neural network interpretability.
\newblock \emph{IEEE Transactions on Emerging Topics in Computational
  Intelligence}, 5\penalty0 (5):\penalty0 726--742, 2021.

\end{thebibliography}
\bibliographystyle{plainnat}

\newpage
\appendix

\section{Missing Proofs}

\subsection{Proof of Lemma~\ref{lem:pig_marginal}}
\begin{proof}
By Taylor's theorem, we have
\[ \nabla g(\ww) = \nabla g(\ww_{\{i\}}) + \oHH \ww_{[n]\backslash\{i\}}\,, \]
where $\oHH$ is an average Hessian on the path from $\ww$ to 
$\ww_{\{i\}}$. Then,
\begin{align*}
\int_{t=0}^1 \nabla_i g(t\onev) dt 
 =
\int_{t=0}^1 \nabla_i g(t\onev_i) dt 
+ \int_{t=0}^1 \onev_i^\top \oHH(t) t(\onev_{[n]\backslash\{i\}}) dt 
\end{align*}
Now, we know that 
$\left|\onev_i^\top \oHH(t)(\onev_{[n]\backslash\{i\}})\right|
\leq K$, and
$\int_{t=0}^1 \nabla_i g(t\onev_i) dt
= g(\onev_i) - g(\zerov)$, so
\begin{align*}
\left|\int_{t=0}^1 \nabla_i g(t\onev) dt 
- (g(\onev_i) - g(\zerov))\right| \leq K/2\,.
\end{align*}
However, by the non-correlation property, this implies that
\[ \nabla_S g(\ww) = \nabla_S g(\ww_S) \]
As a result, we have that
\begin{align*}
\langle \onev, \aa_S \rangle
& = \langle \onev, -\int_{t=0}^1 \nabla_S g(t\onev) dt\rangle \\
& = \langle \onev, -\int_{t=0}^1 \nabla_S g(t\onev_S) dt\rangle \\
& = g(\zerov) - g(\onev_S)\,,
\end{align*}
which completes the proof.
\end{proof}

\subsection{Integrated gradients for linear regression}
\begin{lemma}
We consider the function
$g(\ww) = -\left\|\AA(\xx\odot\ww) - \bb\right\|^2$, where
$\ww\in[0,1]^n$,
$\AA\in\mathbb{R}^{m\times n}$, $\bb\in\mathbb{R}^m$
and $\xx$ is the optimal solution of the linear regression
problem $\min_{\xx}\, \left\|\AA\xx -\bb\right\|^2$.
Then, the integrated gradient scores are given by
\[ \int_{t=0}^1 \nabla g(t\onev) dt = \xx \odot \nabla g(\zerov)\,. \]
\label{lem:pig_linear_regression}
\end{lemma}
\begin{proof}
We note that, for any $t\in[0,1]$, and using the 
known fact that
$\xx = (\AA^\top \AA)^+ \AA^\top \bb$, we have
\begin{align*}
\nabla g(t\onev)
& = 2 \XX^\top \AA^\top (\AA t\xx - \bb) \\
& = 2 t \XX^\top \AA^\top \bb - 2  \AA^\top \bb\\
& = -2(1-t) \XX^\top \AA^\top \bb\\
& = 2(1-t) \xx \odot \nabla g(\zerov)\,,
\end{align*}
where $\XX = \mathrm{diag}(\xx)$. Then, we conclude that
\begin{align*}
\int_{t=0}^1 \nabla g(t\onev) dt
& = 2 \xx \odot \nabla g(\zerov) \int_{t=0}^1 (1-t) dt \\
& = \xx \odot \nabla g(\zerov)\,. \qedhere
\end{align*}
\end{proof}

\section{Experimental Details and Additional Experiments}

\subsection{Explaining GNN predictions}

For these experiments, we used GCN model pretrained on ogbn-arxiv, circa \S\ref{sec:graph_compression}. In these experiments, we want to select graph edges that explain classification of one node (contrast to graph compression \S\ref{sec:graph_compression} where we select edges that maintain classification of \textit{\textbf{all}} nodes).
Given node node $i \in [m]$, we want to select neighbors of $i$, as well as their neighbors, and their neighbors, \dots, up-to the depth of the trained GCN (we used 3 GCN layers), that would make the GCN model prediction unchanged as compared to the full subgraph around node $i$. The gradient orcale was modified to only return nonzero gradients to edges connecting already-discovered nodes.

Figure \ref{fig:explain_gnn} shows a qualitative evaluation of GreedyPIG. One could argue: an article should only cite another if it is related. However, the degree of relatedness can vary. When explaining GNN predictions, we hope the explanation method to select a subgraph that is very related to the center node. The figure shows that random edges around a center node can be less-related to the center node, than if the edges were chosen using GreedyPIG.
In the top row of Fig.~\ref{fig:explain_gnn}, MixHop paper and blue neighbors are  related to GNNs and message passing (MP), whereas red nodes include embedding methods (not MP), or non-GNN applications. In the bottom row, we see that the random edge selection quickly diverged to articles within the NLP domain and otherwise unrelated applications of DeepWalk.

\begin{figure*}
    \centering
    \begin{subfigure}[b]{0.48\textwidth}
        \centering
        \includegraphics[height=2.8cm]{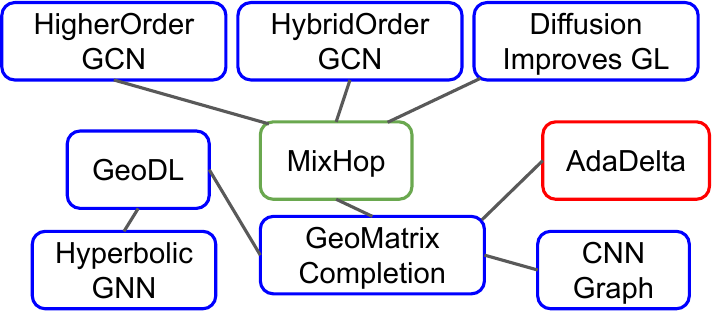}
        \caption{GreedyPIG explains \textit{MixHop} classification.}
    \end{subfigure}
    \begin{subfigure}[b]{0.48\textwidth}
        \centering
        \includegraphics[height=2.8cm]{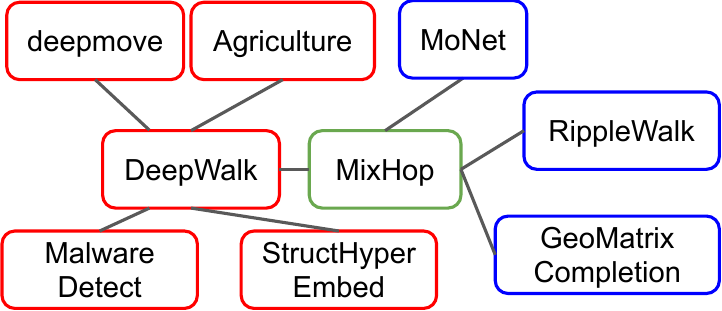}
        \caption{Explaining w/ random edges.}
    \end{subfigure}
    \begin{subfigure}[b]{0.48\textwidth}
        \centering
        \vspace{0.3cm}
        \includegraphics[height=4cm]{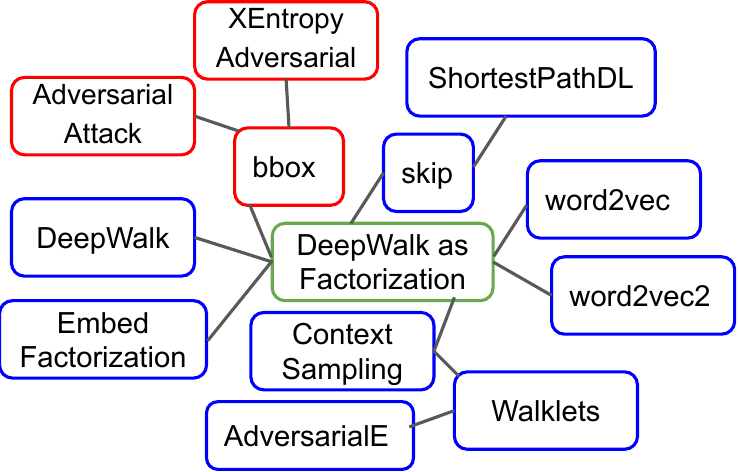}
        \caption{GreedyPIG explains \textit{DeepWalk as ...} classification.}
    \end{subfigure}
    \begin{subfigure}[b]{0.48\textwidth}
        \centering
        \vspace{0.3cm}
        \includegraphics[height=4cm]{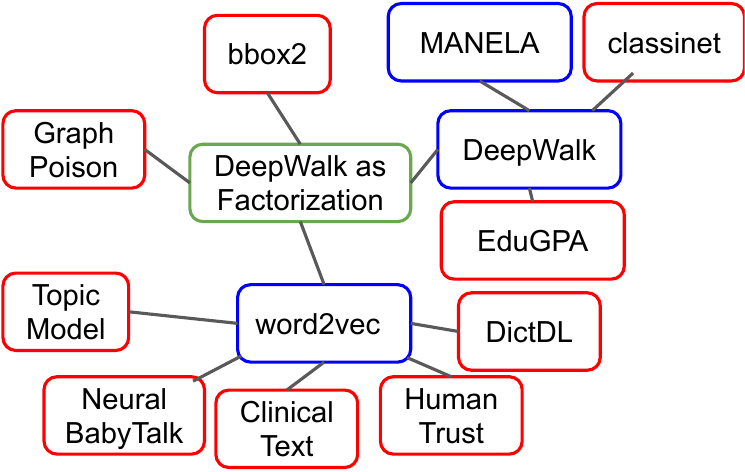}
        \caption{Explaining w/ random edges.}
    \end{subfigure}
    \caption{Pre-trained Graph Neural Network (GNN) is asked to predict article category of center node (green). GreedyPIG (left column) was used to select subgraph around center node that maintains the GNN prediction and we compared it to selecting random adjacent edges. We manually labeled the selected nodes as \textit{strongly-related} (blue) and \textit{less-related} (red).}
    \label{fig:explain_gnn}
\end{figure*}

Finally, we shortcut the paper names to reduce visual clutter. For completeness, the full titles of the papers, as appearing in ogbn-arxiv dataset \citep{ogb}, are as follows.
\begin{itemize}[-]
    \item
    \textbf{DeepWalk as Factorization}: comprehend deepwalk as matrix factorization.
    \item \textbf{MixHop}: mixhop higher order graph convolutional architectures via sparsified neighborhood mixing.
    \item \textbf{DeepWalk}: deepwalk online learning of social representations.
    \item 
  \textbf{Context Sampling}: vertex context sampling for weighted network embedding
  \item 
  \textbf{word2vec}: efficient estimation of word representations in vector space
  \item 
   \textbf{word2vec2}: distributed representations of words and phrases and their compositionality
  \item
  \textbf{skip}:
  don t walk skip online learning of multi scale network embeddings
  \item
  \textbf{bbox}:
    a restricted black box adversarial framework towards attacking graph embedding models 
    \item
    \textbf{AdversarialAttack}:
    threat of adversarial attacks on deep learning in computer vision a survey
  \item
  \textbf{ShortestPathDL}: shortest path distance approximation using deep learning techniques.
  \item
  \textbf{Walklets}: walklets multiscale graph embeddings for interpretable network classification
  \item
  \textbf{AdversarialE}:
  learning graph embedding with adversarial training methods
  \item
  \textbf{Embed Factorization}:
  network embedding as matrix factorization unifying deepwalk line pte and node2vec.
  \item \textbf{bbox}:
  a restricted black box adversarial framework towards attacking graph embedding models
  \item \textbf{XEntropyAdvers}:
  cross entropy loss and low rank features have responsibility for adversarial examples"
  \item
   \textbf{bbox2}:
   the general black box attack method for graph neural networks
  \item
  \textbf{HumanTrust}:
  the transfer of human trust in robot capabilities across tasks.
  \item \textbf{Neural BabyTalk}:
  neural baby talk.
  \item
  \textbf{DictDL}:
  integrating dictionary feature into a deep learning model for disease named entity recognition.
  \item \textbf{MANELA}:
  manela a multi agent algorithm for learning network embeddings.
  \item
  \textbf{ClinicalText}:
  clinical text generation through leveraging medical concept and relations.
  \item
  \textbf{EduGPA}:
  will this course increase or decrease your gpa towards grade aware course recommendation.
  \item
  \textbf{GeoMatrixCompletion}: convolutional geometric matrix completion.
  \item
  \textbf{GeoDL}: geometric deep learning going beyond euclidean data.
  \item
  \textbf{AdaDelta}:
  adadelta an adaptive learning rate method.
  \item
  \textbf{HigherOrderGCN}:
  higher order weighted graph convolutional networks.
  \item
  \textbf{HybridOrderGCN}:
  hybrid low order and higher order graph convolutional networks.
  \item \textbf{CNNGraph}:
  deep convolutional networks on graph structured data.
  \item \textbf{HyperbolicGNN}: hyperbolic graph neural networks.
  \item \textbf{deepmove}: deepmove learning place representations through large scale movement data.
  \item \textbf{Agriculture}: cultivating online question routing in a question and answering community for agriculture.
  \item \textbf{RippleWalk}: ripple walk training a subgraph based training framework for large and deep graph neural network.
  \item \textbf{MalwareDetect}: aidroid when heterogeneous information network marries deep neural network for real time android malware detection.
  \item \textbf{MoNet}: monet debiasing graph embeddings via the metadata orthogonal training unit.
  \item
  \textbf{StructHyperEmbed}:
  structural deep embedding for hyper networks.
\end{itemize}

\subsection{Feature attribution on images}

For the experiments in Section~\ref{sec:feature_attribution}, we 
used the MobilenetV2 image classification network pretrained
on Imagenet, using the implementation from the tf.keras 
library. To implement the algorithms from previous work that
we used in comparisons, i.e. integrated gradients and guided
integrated gradients, we used~\citet{saliency}, which is an
collection of implementations of state of the art saliency
methods. In order to generate Figure~\ref{fig:sic}, we first
took a random sample of Imagenet examples from various classes,
and for each sample we first applied the preprocessing routine
used in MobilenetV2, which includes centering and resizing to
$224\times 224$ pixels with $3$ color channels. This 
$224\times 224 \times 3$ tensor of $3$ dimensions is the input
to the neural network, and as a result it is also the shape of
the attribution map. 

We ran different attribution algorithms 
and sorted the absolute values of attribution values.
Then, we picked $100$ equally spaced values for $k$,
from $0$ to $224\cdot 224\cdot 3$, and 
generated an image using only the top-$k$ attribution values
(and replacing the rest by the baseline, which in this case is
a gray image). For guided PIG and SmoothGrad, we used the default settings from
the saliency library.

\subsection{Pointing game}

In this section, we consider the \emph{pointing game} defined by
\cite{bohle2021convolutional,bohle2022b}, which is a sanity check for
the usefulness of feature attribution methods on images. In this game,
examples of different classes are stitched together in a grid to form
a new image. Then, this new image is fed to the attribution method,
with the goal to explain one of the four classes. Ideally, the
highest attributions should be concentrated in the quadrant that
corresponds to the selected class. In Figures~\ref{fig:pointing_appendix_1} and \ref{fig:pointing_appendix_2}, we 
see one such example that compares the attributions of integrated
gradients and Greedy PIG.

\begin{figure}
\centering
 \begin{minipage}[t]{0.3\linewidth}
 \centering
  \includegraphics[width=\linewidth]{figures/0_223_65_175_264.png}
  \caption*{Input image}
  \label{fig:image1X}
 \end{minipage}
 
 \begin{minipage}[t]{\linewidth}
  \begin{minipage}[t]{.24\linewidth}
   \includegraphics[width=\linewidth]{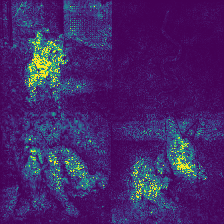}
   \label{fig:image2aX}
  \end{minipage}
  \begin{minipage}[t]{.24\linewidth}
   \includegraphics[width=\linewidth]{figures/pointing_2x2_pig_sea_attributions.png}
   \label{fig:image2bX}
  \end{minipage}
  \begin{minipage}[t]{.24\linewidth}
   \includegraphics[width=\linewidth]{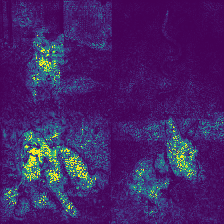}
   \label{fig:image2cX}
  \end{minipage}
  \begin{minipage}[t]{.24\linewidth}
   \includegraphics[width=\linewidth]{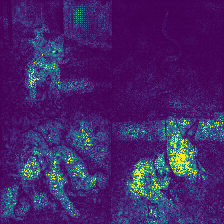}
   \label{fig:image2dX}
  \end{minipage}
  \caption*{Top 15000 Integrated Gradient attributions for each 
  class.}
 \end{minipage}
 
 \begin{minipage}[t]{\linewidth}
  \begin{minipage}[t]{.24\linewidth}
   \includegraphics[width=\linewidth]{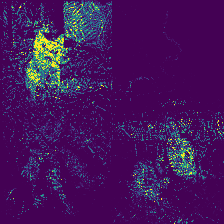}
   \label{fig:image3a}
   \caption*{schipperke}
  \end{minipage}
  \begin{minipage}[t]{.24\linewidth}
   \includegraphics[width=\linewidth]{figures/pointing_2x2_gpig_sea_attributions.png}
   \label{fig:image3b}
   \caption*{sea snake}
  \end{minipage}
  \begin{minipage}[t]{.24\linewidth}
   \includegraphics[width=\linewidth]{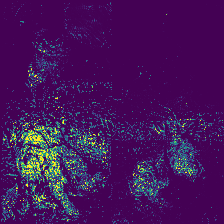}
   \label{fig:image3c}
   \caption*{otterhound}
  \end{minipage}
  \begin{minipage}[t]{.24\linewidth}
   \includegraphics[width=\linewidth]{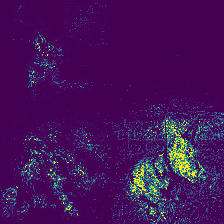}
   \label{fig:image3d}
   \caption*{Cardigan}
  \end{minipage}
  \caption*{Top 15000 Greedy PIG attributions for each 
  class.}
 \end{minipage}
  \caption{The input image is generated as a $2\times 2$ grid of images from
  different classes, here schipperke, sea snake, otterhound and Cardigan.
  Ideally, attributions should be concentrated in the quadrant 
  associated with the respective class.}
\label{fig:pointing_appendix_1}
\end{figure}

\begin{figure}
\centering
 \begin{minipage}[t]{0.3\linewidth}
 \centering
  \includegraphics[width=\linewidth]{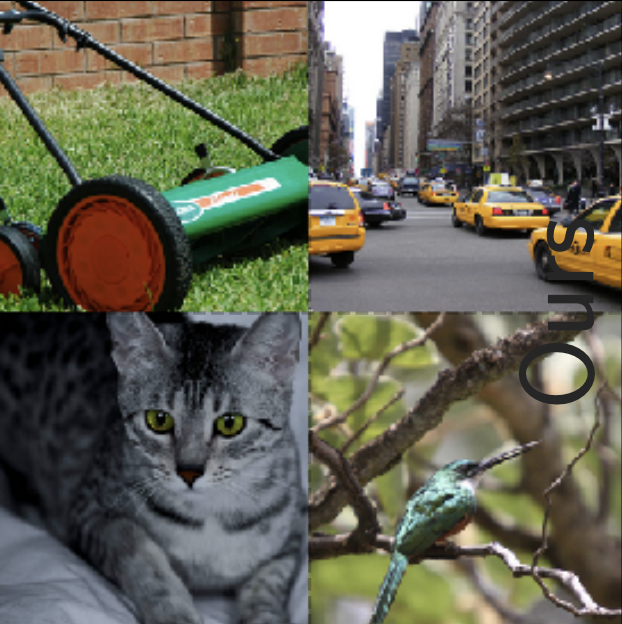}
  \caption*{Input image}
  \label{fig:image1}
 \end{minipage}
 
 \begin{minipage}[t]{\linewidth}
  \begin{minipage}[t]{.24\linewidth}
   \includegraphics[width=\linewidth]{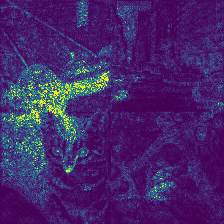}
   \label{fig:image2a}
  \end{minipage}
  \begin{minipage}[t]{.24\linewidth}
   \includegraphics[width=\linewidth]{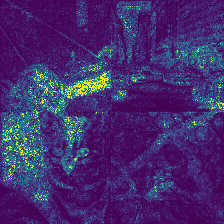}
   \label{fig:image2b}
  \end{minipage}
  \begin{minipage}[t]{.24\linewidth}
   \includegraphics[width=\linewidth]{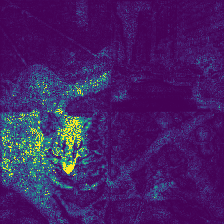}
   \label{fig:image2c}
  \end{minipage}
  \begin{minipage}[t]{.24\linewidth}
   \includegraphics[width=\linewidth]{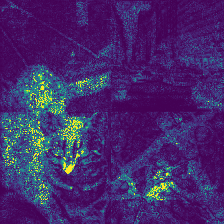}
   \label{fig:image2d}
  \end{minipage}
  \caption*{Top 15000 Integrated Gradient attributions for each 
  class.}
 \end{minipage}
 
 \begin{minipage}[t]{\linewidth}
  \begin{minipage}[t]{.24\linewidth}
   \includegraphics[width=\linewidth]{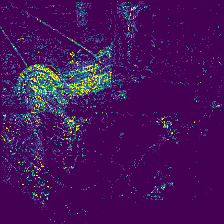}
   \label{fig:image3aX}
   \caption*{lawnmower}
  \end{minipage}
  \begin{minipage}[t]{.24\linewidth}
   \includegraphics[width=\linewidth]{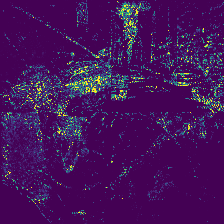}
   \label{fig:image3bX}
   \caption*{cab}
  \end{minipage}
  \begin{minipage}[t]{.24\linewidth}
   \includegraphics[width=\linewidth]{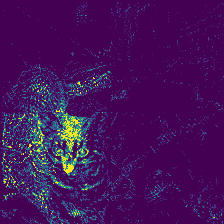}
   \label{fig:image3cX}
   \caption*{Egyptian cat}
  \end{minipage}
  \begin{minipage}[t]{.24\linewidth}
   \includegraphics[width=\linewidth]{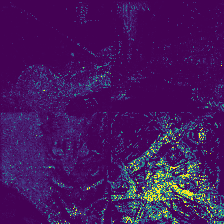}
   \label{fig:image3dX}
   \caption*{jacamar}
  \end{minipage}
  \caption*{Top 15000 Greedy PIG attributions for each 
  class.}
 \end{minipage}
  \caption{The input image is generated as a $2\times 2$ grid of images from
  different classes, here lawnmower, cab, Egyptian cat and jacamar.
  Ideally, attributions should be concentrated in the quadrant 
  associated with the respective class.}
\label{fig:pointing_appendix_2}
\end{figure}

\subsection{Examples}

\begin{figure}[H]
    \centering
    
    \begin{subfigure}[b]{\textwidth}
        \begin{flushright}
        \includegraphics[height=0.2\linewidth,width=0.3\linewidth]{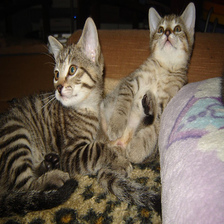}%
        \hspace{0.5cm}
        \includegraphics[height=0.2\linewidth,width=0.3\linewidth]{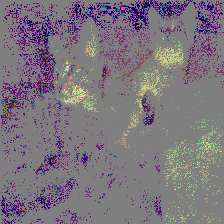}%
        \hspace{0.5cm}
        \includegraphics[height=0.2\linewidth,width=0.3\linewidth]{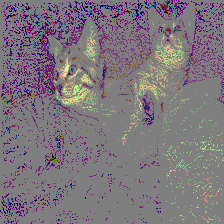}%
        \hfill
        \end{flushright}
        \caption{Top $15000$ 
        features attributed by integrated gradients (left) vs.\
        GreedyPIG (right).
        The softmax output for true class ``tabby''
        \emph{after pruning} unselected features
        is $3\cdot 10^{-5}$ 
        for integrated gradients vs.\ $0.9973$ for GreedyPIG.}
    \end{subfigure}
    \begin{subfigure}[b]{\textwidth}
        \begin{flushright}
            \includegraphics[height=0.2\linewidth,width=0.3\linewidth]{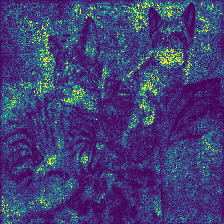}%
            \hspace{0.5cm}
            \includegraphics[height=0.2\linewidth,width=0.3\linewidth]{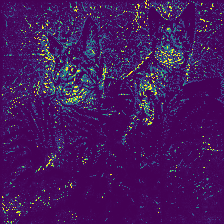}%
            \hfill
        \end{flushright}
        \caption{Heatmap of
        feature attributions by integrated gradients (left)
        vs.\ GreedyPIG (right).
        }
    \end{subfigure}
    \caption{Illustration of attribution algorithms
    on an example from Imagenet labeled ``tabby'' (left column).
    Integrated gradients (IG) (middle)
    ran for $2000$ steps, while GreedyPIG (right)
    ran for $100$ rounds each with $20$ steps.
    Both IG and GreedyPIG maximized softmax objective.}
    \label{fig:tabby}
\end{figure}

\begin{figure}[H]
    \centering
        \includegraphics[width=0.2\linewidth]{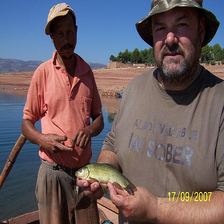}%
        \hspace{0.5cm}
        \includegraphics[width=0.2\linewidth]{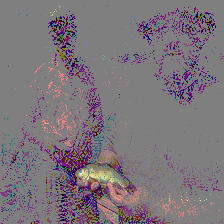}%
        \hspace{0.5cm}
        \includegraphics[width=0.2\linewidth]{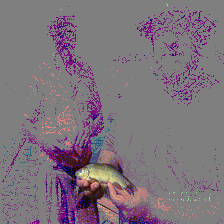}%
    \caption{Illustration of different attribution objectives on an example from Imagenet labeled
    ``tench''.
    Left: Original image. Middle: Top $15000$ 
        feature attributions by Greedy PIG with the softmax objective,
        and Right: with the cross-entropy objective.
    }
    \label{fig:fish}
\end{figure}

\subsection{Block-wise feature attribution}
In the spirit of~\citet{kapishnikov2019xrai}, we augment the greedy PIG 
algorithm with the ability to select patches instead of individual features
from a $3$-dimensional image tensor. In fact, our implementation allows for
specifying arbitrary subset structures that are to be selected as individual features. The results can be found in Figure~\ref{fig:fish2}.

To get into more detail, the approach of~\citet{kapishnikov2019xrai}
is to first compute attribution scores using integrated gradients, and
then iteratively find regions of the image with maximum sum of 
attributions. The main difference from our approach, is that we invoke
the integrated gradient algorithm after \emph{every} block selection,
insteado of just once, in line with our adaptive approach. Specifically,
for a given block size $b\times b$, after each round of greedy PIG we
compute the $b\times b$ patch with the maximum sum of attributions,
and select this patch. It should be noted that selected features have
a score of $0$ in future iterations.

\begin{figure}[H]
    \centering
    \begin{subfigure}[b]{\textwidth}
        \centering
        \includegraphics[width=0.2\linewidth]{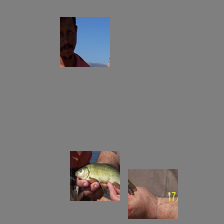}%
        \hspace{0.5cm}
        \includegraphics[width=0.2\linewidth]{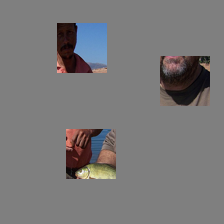}%
        \hspace{0.5cm}
        \includegraphics[width=0.2\linewidth]{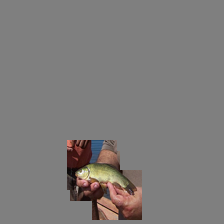}%
        \hspace{0.5cm}
        \includegraphics[width=0.2\linewidth]{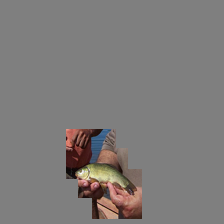}%
        \hfill
        \caption{Block-wise attributions. From left to right: 
       Integrated gradients with softmax,
        integrated gradients with cross entropy,
        greedy PIG with softmax,
        greedy PIG with cross entropy.}
    \end{subfigure}
    \begin{subfigure}[b]{\textwidth}
        \centering
        \includegraphics[width=0.2\linewidth]{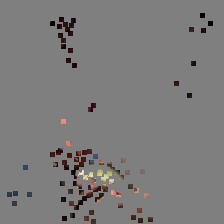}%
        \hspace{0.5cm}
        \includegraphics[width=0.2\linewidth]{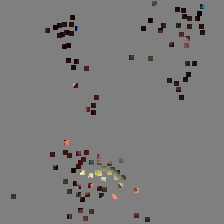}%
        \hspace{0.5cm}
        \includegraphics[width=0.2\linewidth]{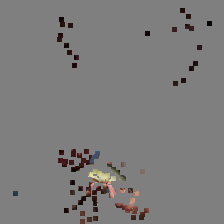}%
        \hspace{0.5cm}
        \includegraphics[width=0.2\linewidth]{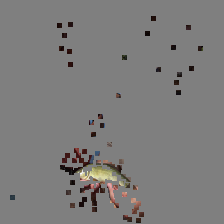}%
        \hfill
        \caption{Block-wise attributions. From left to right: 
       Integrated gradients with softmax,
        greedy PIG with softmax,
        integrated gradients with cross entropy,
        greedy PIG with cross entropy.}
    \end{subfigure}
    \caption{Illustration of different attribution objectives on an example from Imagenet labeled
    ``tench'', with different block sizes.}
    \label{fig:fish2}
\end{figure}
\vspace{-1.00cm}
\begin{figure}[H]
    \centering
    \includegraphics[width=0.45\linewidth]{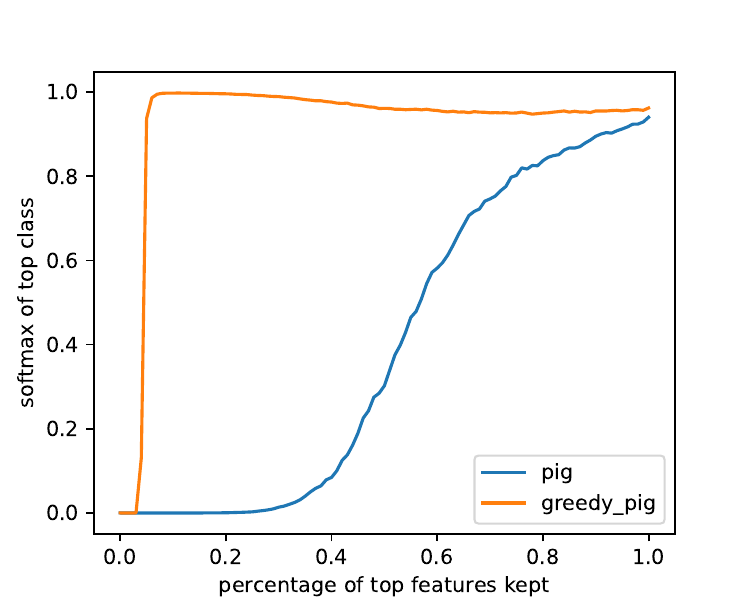}
    \caption{Comparing PIG and Greedy PIG on
    ImageNet attribution using ResNet50.
    The AUC scores for PIG and Greedy PIG are $0.39443$
    and $0.9262$ respectively.}
\end{figure}

\subsection{Feature selection}

For the experiments in Section~\ref{sec:feature_selection}, we used a model
identical to~\citet{yasuda2023sequential}, which is a $3$-hidden layer neural
network with ReLU activations. We implemented minibatch versions of integrated
gradients and greedy PIG. 

\paragraph{Note:} While the attributions computed
by the integrated gradients algorithm can be averaged over the whole dataset
to compute global attribution scores, this is not necessarily the case with the
greedy PIG attribution scores. This is another place where the formulation in
Definition~\ref{def:subset_selection} will be handy. Specifically, given a model
$f(\cdot;\vtheta)$ with parameters $\vtheta$, data $(\XX,\yy)$ and an aggregate 
loss function $\ell$ over the data, 
we define the post-hoc feature selection problem simply by 
$G(S) := - \ell(\yy, f(\XX_S;\vtheta))$.
Note that this is \emph{not} equivalent to averaging the greedy PIG
attributions across the dataset.
While a naive implementation of Algorithm~\ref{alg:greedy_pig} would require multiple
full batch gradient evaluations in each round, we implement a mini-batch version,
which instead samples a number of data batches in each round, and computes the
gradient over these batches.

Specifically, given a number 
$n=150000$ of available data
batches, each of size $512$, 
and a number $g=39$ of gradients to be evaluated (this is the number of
integrated gradient steps), we randomly shuffle the batches and evaluate each
of the gradients on a random set of $n/g\approx 3846$ batches (i.e. the 
gradient is averaged over all these batches). In this way, we can ensure that
all algorithms use the same amount of data, and the algorithms are scalable
enough to be used in large scale settings.

\end{document}